\documentclass{article}

\usepackage{microtype}
\usepackage{graphicx}
\usepackage{subfigure}
\usepackage{booktabs} 
\usepackage{amsmath,amssymb}
\usepackage{amsthm}
\usepackage{nccmath}
\usepackage{here}
\usepackage{subfigure}
\usepackage[all]{xy}
\usepackage{tabularx}

\usepackage{hyperref}

\usepackage[accepted]{icml2020}

\icmltitlerunning{Analysis via ONSs in RKHMs and Applications}

\def\mcl#1{\mathcal{#1}}
\def\blacket#1{\left\langle #1\right\rangle}
\def\hil{\mcl{H}}
\def\nn{\nonumber}
\def\opn{\operatorname}

\def\alg{\mcl{A}}
\def\modu{\mcl{M}}
\def\mat{\mathbb{C}^{m\times m}}
\def\bblacket#1{\big\langle #1\big\rangle}
\def\bblacketg#1{\bigg\langle #1\bigg\rangle}
\def\Bblacket#1{\bigg\langle #1\bigg\rangle}
\def\sblacket#1{\langle #1\rangle}

\newtheorem{prop}{Proposition}[section]
\newtheorem{thm}[prop]{Theorem}
\newtheorem{lemma}[prop]{Lemma}

\newtheorem{defin}[prop]{Definition}

\newtheorem{cor}[prop]{Corollary}
\theoremstyle{definition}
\newtheorem{remark}[prop]{Remark}

\begin{document}

\twocolumn[
\icmltitle{Analysis via Orthonormal Systems in Reproducing Kernel Hilbert $C^*$-Modules and Applications}



\icmlsetsymbol{equal}{*}

\begin{icmlauthorlist}
\icmlauthor{Yuka Hashimoto}{ntt,keiograd}
\icmlauthor{Isao Ishikawa}{riken,keio}
\icmlauthor{Masahiro Ikeda}{riken,keio}
\icmlauthor{Fuyuta Komura}{riken,keiograd}
\icmlauthor{Takeshi Katsura}{riken,keio}
\icmlauthor{Yoshinobu Kawahara}{riken,kyudai}
\end{icmlauthorlist}

\icmlaffiliation{kyudai}{Institute of Mathematics for Industry, Kyushu University, Fukuoka, Japan}
\icmlaffiliation{ntt}{NTT Network Technology Laboratories, NTT Corporation, Tokyo, Japan}
\icmlaffiliation{keiograd}{Graduate School of Science and Technology, Keio University, Kanagawa, Japan}
\icmlaffiliation{keio}{Faculty of Science and Technology, Keio University, Kanagawa, Japan}
\icmlaffiliation{riken}{Center for Advanced Intelligence Project, RIKEN, Tokyo, Japan}

\icmlcorrespondingauthor{Yuka Hashimoto}{yuka.hashimoto.rw@hco.ntt.co.jp}

\icmlkeywords{kernel method, RKHM, structured data, principal component analysis, interacting dynamical system}

\vskip 0.3in
]



\printAffiliationsAndNotice{}  

\begin{abstract}
Kernel methods have been among the most popular techniques in machine learning, where learning tasks are solved using the property of reproducing kernel Hilbert space (RKHS). In this paper, we propose a novel data analysis framework with reproducing kernel Hilbert $C^*$-module (RKHM), which is another generalization of RKHS than vector-valued RKHS (vv-RKHS). Analysis with RKHMs enables us to deal with structures among variables more explicitly than vv-RKHS.
We show the theoretical validity for the construction of orthonormal systems in Hilbert $C^*$-modules, and derive concrete procedures for orthonormalization in RKHMs with those theoretical properties in numerical computations.
Moreover, we apply those to generalize with RKHM kernel principal component analysis and the analysis of dynamical systems with Perron-Frobenius operators.
The empirical performance of our methods is also investigated by using synthetic and real-world data.
\end{abstract}

\section{Introduction}
\label{sec:intro}

Kernel methods have been among the most popular techniques in machine learning (cf.~\citep{scholkopf01}), where learning tasks are solved using the property of reproducing kernel Hilbert space (RKHS). 
RKHS is the space of complex-valued functions equipped with an inner product that is determined with a positive-definite kernel. 
Here, the {\em orthonormality} is defined on the basis of the inner products calculated by evaluating a kernel function, which plays an important role in developing practical algorithms.

Whereas almost all the classical literature on RKHS approaches has focused on complex-valued functions, RKHSs of vector-valued functions, i.e., vector-valued RKHSs (vv-RKHSs), have received increasing attention~\cite{micchelli05,mauricio11,lim15,quang16,kadri16}. This paradigm allows us to incorporate information about structures among observation variables into analyses with RKHSs.
Many kernel methods such as SVM can be generalized to vector-valued functions by employing vv-RKHSs~\cite{quang16}.

However, the use of vv-RKHS may not always be the best option to analyze structured data, mainly for two reasons.
First, although a sample is transformed into a unique function in RKHSs, 
a function in vv-RKHSs corresponding to a sample cannot be determined uniquely. 
Second, although for each sample composed of several elements, relations between all combinations of pairs of the elements should not be lost, information about the relations may be lost and can be difficult to recover when using vv-RKHSs.
This is caused by the fact that an inner product in vv-RKHSs, which measures the similarity between a pair of samples in data, is represented with one complex value.

In this paper, we propose a data analysis framework with reproducing kernel Hilbert $C^*$-module (RKHM), which is another generalization of RKHS.
The theory of RKHM was originally studied in mathematical physics to analyze covariance kernels arising in the theory of generalized stochastic processes and the dilations of operators~\cite{itoh90,heo08}.
Unlike vv-RKHSs, the theory of RKHMs enables us to transform each sample as a unique $C^*$-algebra-valued function, with $C^*$-algebra-valued inner products.
A $C^*$-algebra is a generalization of the space of complex values.
An example of $C^*$-algebra is the space of matrices.
Therefore, we can define a matrix-valued inner product that encodes the similarities of all combinations of pairs of variables in data.

As well as RKHS, the construction of orthonormal systems plays an important role in applying the theory of RKHMs to algorithm developments. 
To the best of our knowledge, although the generalization of the notion of orthonormal to $C^*$-algebra-valued inner-products has been proposed in mathematics~\cite{bakic01}, its application to data analysis has not been discussed so far.
Meanwhile, as for theoretical analyses of RKHMs for data analysis, we can find the only literature by \citet{ye17}, where only the case of RKHMs for data represented as matrices is considered without addressing orthonormality.
Thus, its applicability to data analysis is limited. 
In this paper, we develop novel algorithms using RKHMs through the construction of orthonormal systems with $C^*$-algebra-valued inner-products. 
We first show the theoretical validity for constructing orthonormal systems in Hilbert C*-modules.
Then, we derive concrete procedures for orthonormalization in RKHMs and show some theoretical results on those properties in the numerical computations.
Moreover, we apply those to generalize with RKHM kernel principal component analysis (PCA) and the analysis of dynamical systems with Perron-Frobenius operators, under the motivation described above.
Note that the generalizations are not straightforward because an element of $C^*$-algebra is not always commutative or invertible.
However, 
our generalizations based on the developed theories enable us to access and extract rich information in structured data, which is shown partly here in the context of PCA and the analysis of interacting dynamical systems. 
Note that, to the best of our knowledge, this is the first paper to describe practical algorithms to utilize the notion of orthonormal with respect to $C^*$-algebra-valued inner products in RKHMs for the analysis of structured data.

The remainder of this paper is organized as follows. 
First, in Section~\ref{sec:bg}, we briefly review the theory of RKHMs.
In Section~\ref{sec:orthonormal}, we introduce the definition of the orthonormality with $C^*$-algebra-valued inner products and then show the theoretical validity to consider it.
In Section~\ref{sec:RKHM}, we derive concrete procedures for constructing orthonormal systems in RKHMs with theoretical analyses in their numerical computations.
In Sections~\ref{sec:pca} and \ref{sec:PF}, we develop the applications of the developed results to generalize with RKHM kernel PCA and the analysis of dynamical systems with Perron-Frobenius operators, respectively.
Finally, we empirically investigate the performance of our methods in Section~\ref{sec:results} and conclude the paper in Section~\ref{sec:concl}. 
The notations in this paper are explained in Appendix~\ref{ap:notation}, 
and all proofs are given in Appendix~\ref{ap:proof} in the supplementary material.

\section{Backgrounds}
\label{sec:bg}
In this section, we first briefly review $C^*$-algebras and $C^*$-modules in Section~\ref{sec:c_star_alg}, and RKHMs in Subection~\ref{sec:rkhm_review}.
The precise statements and definitions about them are detailed in Appendix~\ref{ap:rkhm}. 
RKHMs are generalizations of RKHSs.
The theoretical backgrounds of RKHSs are given in Appendices~\ref{ap:review_rkhs}.
We also add a review on vv-RKHSs, 
which are other generalizations of RKHSs, in Appendix~\ref{ap:vv-RKHS_review}.
\subsection{$C^*$-algebra and $C^*$-module}\label{sec:c_star_alg}
A $C^*$-algebra and a $C^*$-module are suitable generalizations of the space of complex numbers $\mathbb{C}$ and a vector space, respectively.  
In this paper, we denote a $C^*$-algebra by $\mathcal{A}$ and a $C^*$-module by $\modu$, respectively. 
As we see below, many complex-valued notions can be generalized to $\mathcal{A}$-valued. 

Mathematically, a $C^*$-algebra is defined as a Banach space equipped with a product structure and an involution  $(\cdot)^*:\mathcal{A}\rightarrow\mathcal{A}$ (Definition~\ref{def:c*_algebra}). 
We denote the norm of $\mathcal{A}$ by $\Vert\cdot\Vert_\mathcal{A}$.  
A main example of $C^*$-algebras 
is the space of matrices $\mat$. 
In this case, the product structure is the usual product of matrices,
the involution is the Hermitian conjugate,
and the norm $\Vert\cdot\Vert_\mathcal{A}$ is the operator norm.

The notion of {\em positiveness} (Definition \ref{def:positive}) is also important in $C^*$-algebras.  
We denote $c\ge0$ if $c\in\mathcal{A}$ is positive. 
In the case of $\mathcal{A}=\mat$, the positiveness corresponds to Hermitian positive semi-definiteness of matrices.  
The notion of positiveness provides us the (pre) order in $\mathcal{A}$, and thus, enables us to consider optimization problems in $\mathcal{A}$.

A {\em $C^*$-module} $\mathcal{M}$ over $\mathcal{A}$ (Definition~\ref{def:c*module}) is  
a linear space $\mathcal{M}$ equipped with a {\em right $\mathcal{A}$-multiplication}.  
For $u\in \mathcal{M},c\in\mathcal{A}$, the right $\alg$-multiplication of $u$ with $c$ is denoted as $uc$. 

A linear map 
$K:\modu\to\modu$ is called an {\em $\mathcal{A}$-linear map} if $K(uc)=(Ku)c$ holds for any $u\in\mathcal{M}$ and $c\in\mathcal{A}$.

\subsection{$C^*$-algebra-valued inner product and RKHM}\label{sec:rkhm_review}
We consider an $\alg$-valued inner product in $\modu$, which is a straightforward generalization of a complex-valued one.
\begin{defin}[$\alg$-valued inner product]\label{def:innerproduct}
A map $\blacket{\cdot,\cdot}:\modu\times\modu\to\alg$ is called an $\alg$-valued {\em inner product} if it satisfies the following properties for $u,v,w\in\modu$ 
and $c,d\in\alg$:

 1. $\blacket{u,vc+wd}=\blacket{u,v}c+\blacket{u,w}d$\;,\\
 2. $\blacket{v,u}=\blacket{u,v}^*$\;,\\
 3. $\blacket{u,u}\ge 0$ and if $\blacket{u,u}=0$ then $u=0$\;.
\end{defin}
For $u\in\modu$, we define the {\em $\alg$-valued absolute value} $\vert u\vert$ on $\modu$ by the positive element  $\vert u \vert$ of $\alg$ such that $\vert u\vert^2=\blacket{u,u}$.   
Since $\vert\cdot\vert$ takes values in more structured space $\alg$, it behaves more complicatedly, but provides us with more information.

The $\alg$-valued absolute value $\vert\cdot\vert$ defines a norm $\Vert \cdot\Vert$ on $\modu$ by $\Vert u\Vert :=\big\Vert\vert u\vert\big\Vert_\alg$. 
By definition, the norm $\Vert\cdot\Vert$ is regarded as an aspect of the $\alg$-valued absolute value $\vert\cdot\vert$.  We call $\modu$ equipped with $\blacket{\cdot,\cdot}$ a {\em Hilbert $C^*$-module} if $\modu$ is complete w.r.t.\@ the norm $\Vert\cdot\Vert$.
An RKHM is an example of Hilbert $C^*$-modules.
Another example is $\alg^n$ (Lemma~\ref{lem:A^m}), which is utilized for orthonormalization (Remark~\ref{rem:A^n}).

We now summarize the theory of RKHMs, which is discussed, for example, in~\cite{heo08}.
Similar to the case of RKHSs, we begin with an $\alg$-valued generalization of a positive definite kernel on a non-empty set $\mcl{Y}$.
\begin{defin}[$\alg$-valued positive definite kernel]\label{def:pdk_rkhm}
 An $\mcl{A}$-valued map $k:\ \mcl{Y}\times \mcl{Y}\to\mcl{A}$ is called a {\em positive definite kernel} if it satisfies the following conditions:
 
 1. $k(x,y)=k(y,x)^*$ \;for $x,y\in\mcl{Y}$,\\
 2. $\sum_{t,s=1}^nc_t^*k(x_t,x_s)c_s\ge 0$ \;for $n\in\mathbb{N}$, $c_1,\ldots,c_{n}\in\alg$, $x_1,\ldots,x_{n}\in\mcl{Y}$.
\end{defin}
We remark that in the case where $\alg$ is the space of bounded linear operators on a Hilbert space, Definition~\ref{def:pdk_rkhm} is equivalent to the operator valued positive definite kernel for the theory of vv-RKHSs (Lemma~\ref{lem:pdk_equiv}).

Let $\phi:\mcl{Y}\to\alg^{\mcl{Y}}$ be the {\em feature map} associated with $k$, which is defined as $\phi(x)=k(\cdot,x)$ for $x\in\mcl{Y}$.
Similar to the case of RKHSs, we construct the following $C^*$-module composed of $\alg$-valued functions by means of $\phi$: 
\begin{equation*}
\modu_{k,0}:=\bigg\{\sum_{t=1}^{n}\phi(x_t)c_t\bigg|\ n\in\mathbb{N},\ c_t\in\alg,\ x_t\in\mcl{Y}\bigg\}.
\end{equation*}
An $\alg$-valued map $\blacket{\cdot,\cdot}_{k}:\modu_{k,0}\times \modu_{k,0}\to\alg$ is defined as follows with $k$:
\begin{equation*}
\bblacketg{\sum_{s=1}^{n}\phi(x_s)c_s,\sum_{t=1}^{l}\phi(y_t)d_t}_{k}:=\sum_{s=1}^{n}\sum_{t=1}^{l}c_s^*k(x_s,y_t)d_t.
\end{equation*}
By the properties in Definition~\ref{def:pdk_rkhm} of $k$, $\blacket{\cdot,\cdot}_{k}$ is well-defined and has the reproducing property.
Also, it satisfies the properties in Definition~\ref{def:innerproduct} (Proposition~\ref{prop:inner_product}).
Therefore, $\blacket{\cdot,\cdot}_{k}$ is an $\alg$-valued inner product.  

The {\em reproducing kernel Hilbert $\alg$-module (RKHM)} associated with $k$ is defined as the completion of $\modu_{k,0}$.  
Similar to the cases of RKHSs, $\blacket{\cdot,\cdot}_{k}$ is extended continuously to the RKHM and 
has the reproducing property (Proposition~\ref{prop:reproducing}).
Also, the RKHM is uniquely determined  (Proposition~\ref{prop:rkhm_unique}).

We denote by $\modu_k$ the RKHM associated with $k$.  
We also denote by $\vert\cdot\vert_k$ and $\Vert \cdot \Vert_k$ the absolute value and norm on $\modu_k$, respectively. 

\section{Orthonormality with $C^*$-algebra-valued Inner Products}
\label{sec:orthonormal}
In this section, we define the orthonormality for $C^*$-algebra-valued inner products, 
and develop theoretical foundations of its validity to apply it to the analysis of structured data.

Orthonormality plays an important role in data analysis because an orthonormal basis constructs orthogonal projections and an orthogonally projected vector minimizes the deviation from its original vector in the projected space.

We begin with, the definition of an orthonormal system and orthonormal basis in a Hilbert $C^*$-module $\modu$.
We refer to, for example, Definition 1.2 in~\cite{bakic01}.
\begin{defin}[Normalized]
A vector $q\in\mcl{M}$ is {\em normalized} if 
 $0\neq\blacket{q,q}=\blacket{q,q}^2$.
\end{defin}
Note that in the case of a general $C^*$-valued-inner product, for a normalized vector $q$ , $\blacket{q,q}$ is not always equal to the identity of $\alg$ in contrast to the case of a complex-valued inner product.
\begin{defin}[Orthonormal system and basis]
Let $\mcl{T}$ be an index set.
A set $\mathcal{S}=\{q_t\}_{t\in\mcl{T}}\subseteq\modu$ is called an {\em orthonormal system (ONS)} of $\mcl{M}$ if 
$q_t$ is normalized for any $t\in\mcl{T}$
and $\blacket{q_s,q_t}=0$ for $s\neq t$. 
We call $\mathcal{S}$ an {\em orthonormal basis (ONB)} if $\mathcal{S}$ is an orthonormal system and dense in $\modu$.
\end{defin}

We show basic properties of the above orthonormal systems. 
First, we derive the existence of an orthonormal basis.
\begin{prop}\label{prop:orthonormal_existance}
The Hilbert $C^*$-module over $\mat$ has an orthonormal basis.
\end{prop}
Unlike Hilbert spaces, Hilbert $C^*$-modules do not always have an orthonormal basis for general $\alg$~\cite{lance95,landi09}.
However, Proposition~\ref{prop:orthonormal_existance} guarantees 
the validity of considering orthonormal bases in RKHMs over $\mat$ for practical applications.
Next, we show a minimization property.
\begin{prop}\label{prop:min_projection}
Let $\{q_t\}_{t\in\mcl{T}}$ be an orthonormal system of $\modu$ and let 
$\mcl{V}$ be the completion of the space spanned by $\{q_t\}_{t\in\mcl{T}}$.
For $w\in\modu_k$, let $P:\modu\to\mcl{V}$ be the projection operator defined as $Pw:=\sum_{t\in\mcl{T}}q_t\blacket{q_t,w}$.
Then $Pw$ is the unique solution of the following minimization problem: 
\begin{equation}
\min_{v\in\mcl{V}}\vert w-v\vert^2\label{eq:min_projection}
\end{equation}
\end{prop}
Unlike the case of complex-valued inner products, the existence of the solution of 
$\alg$-valued
minimization problem~\eqref{eq:min_projection} 
is not obvious since the positive elements in $\alg$ are not totally ordered.
However, Proposition~\ref{prop:min_projection} shows the orthogonally projected vector uniquely minimizes the deviation from an original vector in $\mcl{V}$.

\section{Matrix-valued Positive Definite Kernel and Orthonormality in RKHMs}
\label{sec:RKHM}
In this section, we 
first construct a $\mat$-valued positive definite kernel, which encodes the similarities between all combinations of pairs ($m^2$ pairs) of elements for two samples, each of which is composed of $m$ elements, in Section~\ref{sec:settings}. 
Then, in Section~\ref{sec:gram-schmidt}, we develop theoretical foundations for calculating the orthonormalization in RKHMs.

\subsection{Matrix-valued positive definite kernel}\label{sec:settings}
First, we introduce a matrix-valued positive definite kernel. 
Let $\mcl{X}$ be a topological space.
We consider data composed of $m$ elements of $\mcl{X}$.
Let $\tilde{k}:\mcl{X}\times \mcl{X}\to\mathbb{C}$ be a complex-valued positive definite kernel, $\tilde{\phi}$ be the feature map, and $\hil_{\tilde{k}}$ be the RKHS associated with $\tilde{k}$.
On the basis of the above setting for $\mcl{X}$ and RKHSs, we introduce a $\mat$-valued positive definite kernel $k$, which is also proposed in the framework of vv-RKHSs~\cite{lim15}, to construct an RKHM.
\begin{lemma}\label{lem:pdk}
Let $k:\mcl{X}^m\times \mcl{X}^m\to\mat$ be a matrix valued map where the $(i,j)$-elements of $k(x_1,x_2)$ are defined as $\tilde{k}(x_{1,i},x_{2,j})$ for $x_t=[x_{t,1},\ldots,x_{t,m}]\in\mcl{X}^m$ for $t=1,2$. 
Then, $k$ is a $\mat$-valued positive definite kernel.
\end{lemma}
The $(i,j)$-element of $k(x_1,x_2)$ for $x_1,x_2\in\mcl{X}^m$ equals $\tilde{k}(x_{1,i},x_{2,j})$, 
which represents the similarities between $x_{1,i}$ and $x_{2,j}$ in $\hil_{\tilde{k}}$.
Thus, the inner product between $\phi(x_1)$ and $\phi(x_2)$ describes the similarities of all combinations of pairs of elements of $x_1$ and $x_2$.


\subsection{Gram-Schmidt orthonormalization}\label{sec:gram-schmidt}
We now consider the practical approach for orthonormalization with matrix-valued inner products.
\begin{prop}[Normalization]\label{prop:normalized_property}
Let $\epsilon\ge 0$ and let $\hat{q}\in\mcl{M}$ satisfy $\Vert \hat{q}\Vert>\epsilon$.
Then, there exists $\hat{b}\in\mat$ such that $\Vert \hat{b}\Vert_{\mat}<1/\epsilon$ and $q:=\hat{q}\hat{b}$ is normalized.
In addition, there exists $b\in\mat$ such that 
$\Vert \hat{q}-qb\Vert\le\epsilon$.
\end{prop}\vspace{-.3cm}
\begin{proof}[Sketch of the proof]
Let $\lambda_1\ge\ldots\ge\lambda_{m}\ge 0$ be the eigenvelues of $\blacket{\hat{q},\hat{q}}$.
Since $\blacket{\hat{q},\hat{q}}$ is positive, there exists an unitary matrix $c$ such that $\blacket{\hat{q},\hat{q}}=c^*\opn{diag}\{\lambda_1,\ldots,\lambda_{m}\}c$.
Let $m':=\max\{j\mid\ \lambda_j>\epsilon^2\}$ and let $\hat{b}:=c^*\opn{diag}\{1/\sqrt{\lambda_1},\ldots,1/\sqrt{\lambda_{m'}},0,\ldots,0\}c$.
The existance of $m'$ follows by the inequality $\Vert \hat{q}\Vert>\epsilon$.
Then, it can be shown that $q:=\hat{q}\hat{b}$ is normalized.
In addition, let $b:=c^*\opn{diag}\{\sqrt{\lambda_1},\ldots,\sqrt{\lambda_{m'}},0,\ldots,0\}c$.
Then, it can be shown that $\Vert \hat{q}-qb\Vert\le\epsilon$ holds.
\end{proof}
Proposition~\ref{prop:normalized_property} and its proof provide a concrete procedure to obtain normalized vectors in $\modu$ in practical situations.
This enables us to compute an orthonormal basis practically by applying Gram-Schmidt orthonormalization, which is introduced in~\cite{cnops92} in a theoretical framework.
\begin{prop}[Gram-Schmidt orthonormalization]\label{prop:gram-schmidt}
Let $\{w_t\}_{t=1}^{\infty}$ be a sequence in $\modu$.
Consider the following scheme for $t=1,2,\ldots$ and $\epsilon\ge 0$:
\begin{equation}
 \begin{aligned}
 \hat{q}_t&=w_t-\sum_{s=1}^{t-1}q_s\blacket{q_s,w_t},\\
q_t&=\hat{q}_t\hat{b}_t\quad \mbox{if }\;\Vert \hat{q}_t\Vert>\epsilon,\quad q_t=0\quad\mbox{o.w.},
\end{aligned}\label{eq:gram-schmidt}
\end{equation}
where $\hat{b}_t$ is defined as $\hat{b}$ in Proposition~\ref{prop:normalized_property} by setting $\hat{q}=\hat{q}_t$.
Then, $\{q_t\}_{t=1}^{\infty}$ is an orthonormal system in $\modu$ such that any $w_t$ is contained in the $\epsilon$-neighborhood of the space spanned by $\{ q_t\}_{t=1}^{\infty}$.
\end{prop}
\begin{cor}\label{cor:gram-schmidt}
If $\epsilon=0$, and the space spanned by $\{w_t\}_{t=1}^{\infty}$ is dense in $\modu$, then $\{q_t\}_{t=1}^{\infty}$ is an orthonormal basis of $\modu$.
\end{cor}
\begin{remark}\label{rem:trade_off}
We give some remarks about the role of $\epsilon$ in Proposition~\ref{prop:normalized_property}, Proposition~\ref{prop:gram-schmidt} and Corollary~\ref{cor:gram-schmidt}.  
$\hat{q}_t$ can always be reconstructed by $w_t$ only when $\epsilon=0$.
This is because the information of the eigenvalues of $\blacket{\hat{q}_t,\hat{q}_t}_k$ may be lost if $\epsilon>0$.
However, if $\epsilon$ is sufficiently small, 
we can reconstruct $\hat{q}_t$ with a small error.
On the other hand, the norm of $\hat{b}_t$ can be large if $\epsilon$ is small, and the computation of $\{q_t\}_{t=1}^{\infty}$ can become numerically unstable.
This corresponds to the trade-off between the theoretical accuracy and numerical stability.
We will also confirm the trade-off empirically in Section~\ref{sec:exp_tradeoff}.
\end{remark}

In practical computations, scheme~\eqref{eq:gram-schmidt} should be represented with matrices.
For this purpose, we derive the following about QR decomposition from Proposition~\ref{prop:gram-schmidt}.
\begin{cor}[QR decomposition]\label{prop:qr}
For $n\in\mathbb{N}$, let $W:=[w_1,\ldots,w_{n}]$ and $Q:=[q_1,\ldots,q_{n}]$.
Let $\epsilon\ge 0$ and $\mathbf{R}:=[r_{s,t}]_{s,t}$ be an $n\times n$ $\mat$-valued matrix.
Here, $r_{s,t}$ is in $\mat$, and defined by $r_{s,t}=\blacket{q_s,w_t}$ for $s<t$, $r_{s,t}=0$ for $s>t$, and $r_{t,t}=b_t$, where $b_t$ is defined as $b$ in Proposition~\ref{prop:normalized_property} by setting $\hat{q}=\hat{q}_t$.
In addition, let $\hat{\mathbf{B}}:=\opn{diag}\{\hat{b}_1,\ldots,\hat{b}_{n}\}$, $\mathbf{B}:=\opn{diag}\{{b}_1,\ldots,{b}_{n}\}$ and $\mathbf{R}_{\opn{inv}}:=\mathbf{\hat{B}}(I+(\mathbf{R}-\mathbf{B})\mathbf{\hat{B}})^{-1}$ be $n\times n$ $\mat$-valued matrices.
Then, the following relations are derived: 
\begin{equation}
Q=W\mathbf{R}_{\opn{inv}},\quad
\Vert W-Q\mathbf{R}\Vert\le \epsilon.\label{eq:qr_dec}
\end{equation}
\end{cor}
Decompositions~\eqref{eq:qr_dec} are called QR decompositions.
Note that the projection operator onto 
the space spanned by $\{q_t\}_{t=1}^n$ is represented as $QQ^*$, where $Q^*$ is the adjoint operator of $Q$, which maps $v\in\modu$ to $[\blacket{q_1,v},\ldots,\blacket{q_{n},v}]\in(\mat)^n$.
The pseudo-code of the QR decomposition is shown in Algorithm~\ref{al:qr}. 
\begin{remark}\label{rem:A^n}
If a vector $u\in\modu$ is represented as $u=W\mathbf{v}$ for some $\mathbf{v}\in(\mat)^n$, $\mathbf{v}$ describes the coordinate of $u$ with respect to $\{w_t\}_{t=1}^n$.
Then, $\mathbf{R}_{\opn{inv}}$ and $\mathbf{R}$ are $\mat$-linear maps on $(\mat)^n$, which are regarded as coordinate transformation matrices.
\end{remark}
\begin{remark}\label{rem:matrix}
In practice, we only have to compute $\mathbf{R}_{\opn{inv}}$ and $\mathbf{R}$ although we are treating vectors in an infinite dimensional space $\modu$. 
In other words,
all calculations in an algorithm are reduced to $n\times n$ $\mat$-valued matrix calculations.
Note that calculations with $n\times n$ $\mat$-valued matrices are regarded as the block calculation of $mn\times mn$ complex-valued ones, which allows us to implement our methods with standard matrix calculations.
\end{remark}

\section{PCA with RKHMs}
\label{sec:pca}
Here, we generalize kernel principal component analysis (kernel PCA) 
to RKHMs.
Since $\phi(x)\in\modu_k$ for $x\in\mcl{X}$ is a function that returns the matrix that encodes similarities with all elements of $x$
, we can implement PCA with taking the similarity between all combinations of pairs of elements into consideration by generalizing it to RKHMs. 
Kernel PCA with RKHSs is briefly reviewed in Appendix~\ref{ap:pca_review}.
We describe the generalization of kernel PCA to RKHMs in Section~\ref{sec:pca_alg} and then its theoretical analysis in Section~\ref{sec:pca_analysis}.

\subsection{Generalization of PCA to RKHM}\label{sec:pca_alg}
First, we construct principal axes and components. 
Let $x_1,\ldots,x_{n}\in\mcl{X}^m$ be samples of structured data, 
$w_t:=\phi(x_t)$ be the sample embedded in an RKHM 
for $t=1,\ldots,n$, 
$W=[w_1,\ldots,w_n]$ be the operator composed of the samples, and $\mathbf{G}:=[\blacket{w_s,w_t}_k]_{s,t}$ be the $n\times n$ $\mat$-valued Gram matrix.
Let $\mathbf{G}=\mathbf{V\Sigma V^*}$ be an eigenvalue decomposition of 
$\mathbf{G}$ with regarding $\mathbf{G}$ as an $mn\times mn$ complex-valued matrix (Remark~\ref{rem:matrix}).
Here, $\mathbf{\Sigma}:=\opn{diag}\{\sigma_1,\ldots,\sigma_l\}$ and $\sigma_1\ge\ldots\ge\sigma_l> 0$ is the nonzero eigenvalues of $\mathbf{G}$.
In addition, let $\mathrm{v}_s$ be the $s$-th column of $\mathbf{V}$.
Using $\mathbf{\Sigma}$ and $\mathbf{V}$, we represent the samples in the smallest possible space.
For this purpose, we define the $s$-th {\em principal axis} for $s=1,\ldots,l$ as follows:
\begin{equation*}
p_s:=\sigma_s^{-1/2}W[\mathrm{v}_s,\underbrace{0,\ldots,0}_{m-1}].
\end{equation*}
\begin{prop}\label{prop:rank1proj}
 $\{p_t\}_{t=1}^s$ is an orthonormal basis of 
 the space spanned by $\{w_t\}_{t=1}^n$,
 and $\blacket{p_s,p_s}_k$ is rank-one.
\end{prop}

Therefore, for $s<l$, we project each $\phi(x_t)$ onto the space spanned by $\{p_t\}_{t=1}^s$.
The projected vector is represented as the sum of $p_j\blacket{p_j,\phi(x_t)}_k$, which is called the $j$-th {\em principal component} of $\phi(x_t)$.
The pseudo-code for computing the principal components is given in Algorithm~\ref{al:pca}. 
We show these principal axes and components are generalizations of the ones that appear in the standard kernel PCA.
\begin{prop}\label{prop:pca_generalization}
 If $m=1$, our principal axis and components are equal to those of kernel PCA with RKHSs. 
\end{prop}

\subsection{Theoretical analysis of kernel PCA with RKHMs}\label{sec:pca_analysis}
We theoretically analyze our kernel PCA with RKHMs described above.
We show that it is interpreted as finding a subspace in RKHMs where samples are projected so that it minimizes a reconstruction error, which is analogous to the standard PCA~\citep[Section 14]{scholkopf01}.

A reconstruction error is caused by projecting samples onto some subspace,
and is represented as
\begin{equation*}
\sum_{t=1}^n\bigg\vert w_t-\sum_{j=1}^sp_j\blacket{p_j,w_t}_k\bigg\vert_k^2\in\mat,
\end{equation*}
in our case. In the case of RKHSs, the reconstruction errors are equal to the sum of the smallest $n-s$ eigenvalues of the Gram matrices.
The analogy for our case is given by considering the trace of the above matrix-valued reconstruction error, since the trace of a matrix equals the sum of its eigenvalues.
Thus, for $s=1,\ldots,l$, we find solutions of the following minimization problem:
\begin{equation}
\min_{\substack{\{\hat{p}_j\}_{j=1}^s\subseteq\modu_k:\ \mbox{\tiny ONS}, \\\blacket{\hat{p}_j,\hat{p}_j}_k:\ \mbox{\tiny rank-one}}}\!\! \opn{tr}\bigg(\sum_{t=1}^n\Big\vert w_t-\sum_{j=1}^s\hat{p}_j\blacket{\hat{p}_j,w_t}_k\Big\vert_k^2\bigg).\label{eq:pca_max}
\end{equation}
In the same manner as kernel PCA with RKHSs~\citep[Proposition 14.1]{scholkopf01}, the following theorem shows the principal axes minimize Eq.~\eqref{eq:pca_max}:
\begin{thm}\label{prop:pca_ineq}
$\{p_t\}_{t=1}^s$ minimizes Eq.\eqref{eq:pca_max} for $s$$=$$1,$$\ldots$$,$$l$.
\end{thm}

We can also consider the centered version of our kernel PCA with RKHMs by replacing $w_t$ with $w_t=\phi(x_t)-1/n\sum_{j=1}^n\phi(x_j)$.
In this case, it can be shown that $\{p_t\}_{t=1}^s$ maximizes the variance of $\phi(x_1),\ldots,\phi(x_n)$.

\section{Analysis of Interacting Dynamical Systems with RKHMs}
\label{sec:PF}

The problem of analyzing dynamical systems from data by using Perron-Frobenius operators and their adjoints (called Koopman operators), which are linear operators expressing the time evolution of dynamical systems, has recently attracted attention in various fields~\cite{mezic12,mezic17,takeishi17,takeishi17-2,lusch17,klus19}.
And, several methods for this problem using RKHSs have also been proposed~\cite{kawahara16,klus17,ishikawa18,hashimoto19},
In these methods, sequential data is supposed to be generated from dynamical systems and is analyzed through those corresponding representations with Perron-Frobenius operators in RKHSs.
Also, as for interacting dynamical systems, \citet{fujii19} proposed a method for estimating linear relations between matrices that describe relations between all combinations of pairs of observables at time $t$ and $t+1$.
We briefly review the existing methods for this approach in Appendix~\ref{sec:PF_review}.

In this section, we propose a generalized method with RKHMs for the analysis with Perron-Frobenius operators for cases where multiple dynamical systems interact, which often occurs in various dynamic phenomena around us. 
Note that information about such interaction may be lost with RKHSs since inner products in RKHSs are complex-valued.
We first generalize the Perron-Frobenius operators to RKHMs in Section~\ref{sec:PF_rkhm} and then consider 
prediction errors and modal decompositions for the generalized operators, respectively, in Sections~\ref{sec:prediction_err} and~\ref{sec:extract}.

\subsection{Perron-Frobenius operators in RKHMs}\label{sec:PF_rkhm}
We generalize the Perron-Frobenius operators in RKHSs (summarized in Appendix~\ref{sec:PF_review}) to those in RKHMs.

First, let $\mcl{Y}:=\{x_0,x_1,\ldots\}\subseteq\mcl{X}^m$ be observed data, where $x_t=[x_{t,1},\ldots,x_{t,m}]$.
And, consider the following interacting dynamical system:
\begin{equation*}
x_{t+1,i}=f_i(x_{t,1},\ldots,x_{t,m})\quad (i=1,\ldots,m),
\end{equation*}
where $f_i:\mcl{X}^m\to\mcl{X}$ is a (possibly, nonlinear) map.
In an RKHM $\modu_k$, we define the Perron-Frobenius operator $K:\modu_{k,0}(\mcl{Y})\to\modu_k$ in the same manner as those in RKHSs as follows:
For $x\in\mcl{Y}$, we consider an operator 
\begin{equation*}
K\phi(x):=\phi([f_1(x),\ldots,f_{m}(x)]),
\end{equation*}
that describs the time evolution of the dynamical system.
Here, $\modu_{k,0}(\mcl{Y}):=\{\sum_{t=0}^n \phi(x_t)c_t\mid\ n\in\mathbb{N},\ x_t\in\mcl{Y},\ c_t\in\mat\}$.
Note that
$\modu_{k,0}(\mcl{Y})$ is dense in $\modu_k$ if $\mcl{Y}$ is dense in $\mcl{X}$.
In addition, $K$ can be extended to $\modu_{k,0}(\mcl{Y})$ as a $\mat$-linear map if $\{\phi(x)\mid\ x\in\mcl{Y}\}$ is $\mat$-linearly independent.
We remark that we sometimes need a fine argument to extend $K$ because of the rank deficiency of 
the matrix-valued positive definite kernel. 
Its mathematical treatments are detailed in Appendix~\ref{ap:lin_indep}.

Under the above preparation, we now estimate $K$ with finite observables $x_0,\ldots,x_T\in\mcl{Y}$.
To obtain the estimation, 
we consider the following minimization problem:
\begin{equation}
 \min_{\hat{K}\in\mcl{L}(\mcl{W}_T)}\sum_{t=0}^{T-1}\big\vert \hat{K}\phi(x_t)-\phi(x_{t+1})\big\vert_k^2,\label{eq:min_krylov}
\end{equation}
whose solution $\hat{K}$ well approximates $K$. 
Here, 
$\mcl{W}_T$ is the space spanned by $\{\phi(x_t)\}_{t=0}^{T-1}$ and $\mcl{L}(\mcl{W}_T)$ is the space of all $\mat$-linear maps on $\mcl{W}_T$.
Existence of a solution of problem~\eqref{eq:min_krylov} follows from Proposition~\ref{prop:min_projection}. 
Thus, we utilize the QR decomposition described in Corollary~\ref{prop:qr} to obtain an explicit representation of the solution as follows:\@
for $t=0,1,\ldots$ and $\epsilon\ge 0$, 
let $\{q_t\}_{t=0}^{\infty}$ be the orthonormal system obtained by setting $w_t$ in the scheme \eqref{eq:gram-schmidt} as $\phi(x_t)$ .
Then, $Q_T=W_T\mathbf{R}_{\opn{inv},T}$ holds, where $W_T:=[\phi(x_0),\ldots,\phi(x_{T-1})]$, and $Q_T$ and $\mathbf{R}_{\opn{inv},T}$ are defined as $Q$ and $\mathbf{R}_{\opn{inv}}$ in Corollary~\ref{prop:qr}.
As a result, the solution of problem~\eqref{eq:min_krylov} is explicitly represented as follows: 
\begin{thm}\label{prop:min_sol}
If $\epsilon=0$ and $\{\phi(x_t)\}_{t=0}^{T-1}$ is linearly independent,  
$Q_TQ_T^*KQ_TQ_T^*$ is the unique solution of problem~\eqref{eq:min_krylov}.
Also, $Q_T^*KQ_T=Q_T^*[\phi(x_1),\ldots,\phi(x_{T})]\mathbf{R}_{\opn{inv},T}$ holds.
\end{thm}
\begin{remark}
Let $\mathbf{K}_T=Q_T^*KQ_T$. 
Then, $\mathbf{K}_T$ is regarded as a matrix representation of $Q_TQ_T^*KQ_TQ_T^*$ with respect to the orthonormal basis $\{q_t\}_{t=0}^{T-1}$.
Since $\mathbf{K}_T=Q_T^*[\phi(x_1),\ldots,\phi(x_{T})]\mathbf{R}_{\opn{inv},T}$ holds,
$\mathbf{K}_T$ can be computed only with finite observables.
\end{remark}

We derive the following proposition about the convergence of $Q_TQ_T^*KQ_TQ_T^*$.
\begin{prop}\label{prop:pf_conv}
If $\epsilon=0$, $\tilde{k}$ defined in Section~\ref{sec:settings} is continuous, and $\mcl{Y}$ is dense in $\mcl{X}^m$, and 
if $K$ is bounded, then $Q_TQ_T^*KQ_TQ_T^*$ converges strongly to $K$ as $T\to\infty$.
\end{prop}

\subsection{Evaluation of Prediction Errors}\label{sec:prediction_err}
Here, we discuss an evaluation of prediction accuracy with the estimated operator $\mathbf{K}_T$.
We generalize the procedure in Section~6 in \cite{hashimoto19} to RKHMs, and define a matrix-valued prediction error.

\citet{hashimoto19} proposed evaluating a prediction error with Perron-Frobenius operators in RKHSs 
with maximal mean discrepancy (MMD) (cf.\@ Eq.~\eqref{eq:error_rkhs} in Appendix~\ref{sec:PF_review}).
However, since this error is real-valued, it does not provide information about which elements are deviated from the prediction.
Whereas, as mentioned above, $\phi(x)\in\modu_k$ for $x\in\mcl{X}^m$ encodes similarities between all elements of $x$.
Thus, by the generalization, 
we can define matrix-valued prediction errors with taking the similarities between all combinations of pairs of elements of data into account.

To address this, we generalize the real-valued prediction error to a matrix-valued absolute value at time $S$ as
\begin{equation}
\hat{a}_{T,S}:=\big\vert \phi(x_S)-Q_T{\mathbf{K}}_TQ_T^*\phi(x_{S-1})\big\vert_{k}^2~\in\mat.\label{eq:error_rkhm}
\end{equation}
Since this error is matrix-valued, we can extract the error with respect to each interaction among the combinations in variables of $x$.
Indeed, the following proposition shows each diagonal element of prediction error~\eqref{eq:error_rkhm} corresponds to the prediction error of each element of $x_S\in\mcl{X}^m$.
\begin{prop}\label{prop:abnormality}
Let $c_t\in\mat$ for $t=0,\ldots T-1$ 
satisfy $Q_T{\mathbf{K}}_TQ_T^*\phi(x_{S-1})=\sum_{t=0}^{T-1}\phi(x_t)c_t$.
Then, the $(j,j)$ element of matrix-valued error~\eqref{eq:error_rkhm} equals $\big\Vert\sum_{t=0}^{T-1}\sum_{i=1}^{m}(c_t)_{i,j}\tilde{\phi}(x_{t,i})-\tilde{\phi}(x_{S,j})\big\Vert_{\tilde{k}}^2$.  
Here, $\tilde{k}$ and $\tilde{\phi}$ is defined in Section \ref{sec:settings}.
\end{prop}

\subsection{Modal Decomposition}\label{sec:extract}
We introduce a notion of eigenpairs of the estimated operator $\mathbf{K}_T$ and give a method for extracting relations invariant with respect to time.
This method is applicable to, for example, causal estimation of time-series data.
The notions of eigenvalues, eigenvectors, and diagonalization in a Hilbert $C^*$-module have been considered for orthonormal eigenvectors~\cite{kadison83,frank95}.
However, in our case, eigenvectors are not necessarily orthnormal.
Therefore, we extend their definition.
An eigenpair of $\mathbf{K}_T$ is a pair $(a,v)$ that satisfies $\mathbf{K}_Tv=va$.
Mathematically, by taking the non-commutativeness and rank deficiency into account, eigenpairs of a $\mat$-linear map on a Hilbert $C^*$-module $\modu$ are defined as follows:
\begin{defin}[Eigenpair]
Let $K$ be a $\mat$-linear map on $\modu$.
Let $(\mat)^{\times}$ be the set of invertible matrices.
We define an {\em eigenpair} of $K$ as an equivalence class of the quotient set $\{(a,v)\mid\ \lambda\in\mat,\ v\in\modu,\ a_{\mid\mcl{V}^{\perp}}=0,\ Kv=va\}/\sim$, where $\sim$ is defined as
\begin{equation*}
(a,v)\sim(b,u)\ \Leftrightarrow\ ^{\exists}c\in(\mat)^{\times},\ b=c^{-1}a c,\ u=vc.
\end{equation*}
\end{defin}

Although the mathematical definition is slightly complicated,
in practical situations, eigenpairs of $\mathbf{K}_T$ are easy to find as follows:\@
Let $\lambda_1,\ldots,\lambda_{mT}\in\mathbb{C}$ be the eigenvalues of $\mathbf{K}_T$  and $\mathrm{v}_1,\ldots,\mathrm{v}_{mT}\in\mathbb{C}^{mT}$ be the eigenvectors of $\mathbf{K}_T$ with respect to $\lambda_1,\ldots,\lambda_{mT}$.
Here, $\mathbf{K}_T$ is regarded as an $mT\times mT$ complex-valued matrix (Remark~\ref{rem:matrix}).
We set $\mathbf{v}_t:=[\mathrm{v}_{t},0,\ldots,0]\in(\mat)^T$ and 
\begin{equation*}
a_t:=\opn{diag}\{\lambda_{t},0,\dots,0\}\in\mat.
\end{equation*}
Then, we can see, for $t=1,\ldots,mT$, the pair $(a_t,\mathbf{v}_t)$ is a representative of an eigenpair.
Since the relation $\phi(x_s)=K^s\phi(x_0)$ holds for time $s$, we approximate 
$\phi(x_s)$ as $Q_T\mathbf{K}_T^sQ_T^*\phi(x_0)$, and apply the above eigenpairs for extracting time-invariant relations.
\begin{prop}\label{prop:eigsum}
Assume $[\mathrm{v}_{1},\ldots,\mathrm{v}_{mT}]\in\mathbb{C}^{mT\times mT}$ is invertible. 
Let $c_t\in\mat$ satisfy $Q_T^*\phi(x_0)=\sum_{t=1}^{mT}\mathbf{v}_tc_t$.
Then, 
$\vert Q_T\mathbf{K}_T^sQ_T^*\phi(x_0)\vert_k^2$
equals the following sequence:
\begin{align}
&\sum_{t,l=1}^{mT}c_t^*(a_t^*)^s\blacket{\mathbf{v}_t,\mathbf{v}_{l}}a_{l}^sc_{l}\label{eq:eigsum}.
\end{align}
Let $c_{\opn{inv}}:=\sum_{t\in\mcl{T}}c_t^*(a_t^*)^s\blacket{\mathbf{v}_t,\mathbf{v}_{t}}a_{t}^sc_t$, where $\mcl{T}:=\{t\mid\ \vert \lambda_t\vert=1\}$.
Then, $c_{\opn{inv}}$ is invariant with respect to $s$.
\end{prop}
If the $(i,j)$ element of $c_{\opn{inv}}$ is large, the $i$-th and $j$-th elements of $x_s$ are similar for arbitrary $s$.
This is because the $(i,j)$ element of $\vert\phi(x_s)\vert_k^2=\blacket{\phi(x_s),\phi(x_s)}_k=k(x_s,x_s)$ represents the similarity between the $i$-th and $j$-th elements of $x_s$.
As a result, we can extract the information of the similarities that are invariant with respect to time.    

\section{Experimental Evaluations}
\label{sec:results}
Here, we show some empirical results by our methods with RKHMs. 
In Section~\ref{sec:exp_tradeoff}, we first empirically investigate the trade-off between the theoretical accuracy and numerical stability described in Remark~\ref{rem:trade_off}. Then, we empirically illustrate the behavior of PCA with RKHMs in Section~\ref{sec:exp_pca}, and the analysis of time-series data with Perron-Frobenius operators in RKHMs in Section~\ref{sec:exp_PF}. 
All the experiments were implemented with Python 3.7, and we used the Laplacian kernel, $\tilde{k}(x,y)=e^{-\Vert x-y\Vert_1}$ for $x,y\in\mathcal{X}$, for constructing $\mat$-valued positive definite kernel $k$.

\subsection{Trade-off between the theoretical accuracy and numerical stability}
\label{sec:exp_tradeoff}
We evaluated the trade-off in Remark~\ref{rem:trade_off} on the basis of $\hat{a}_{T,S}$ defined in Eq.~\eqref{eq:error_rkhm}.
We used $100$ synthetic sequences randomly generated by the following interacting dynamical system with $\mcl{X}\subseteq\mathbb{R}$ for $i=1,\ldots,m$:
\begin{equation}
x_{t+1,i}=1+\frac1n\sum_{j=1}^n(x_{t,i}-x_{t,j})+\xi_{t,i},\label{eq:intercting_dyn}
\end{equation}
where $x_0=[0,\ldots,m]/m$, $m=50$ and $\xi_{t,i}$ is a random number generated from the Gaussian distribution with mean $0$ and standard deviation $0.01$ for $i=1,\ldots,50$ and $t=1,\ldots,T$.
Figure~\ref{fig:tradeoff} shows the averaged value of the criterion 
\begin{equation}
\Vert \hat{a}_{T+1,S}-\hat{a}_{T,S}\Vert_{\mat},\label{eq:a_conv}
\end{equation}
for $T=0,\ldots,19$, $S=30$, and $\epsilon^2=10^{-8},10^{-5},10^{-2}$. In this evaluation, we used the 32 bit floating-point arithmetic, instead of the 64 bit one, to emphasize the numerical stability.
Theoretically, if $\epsilon^2=0$ and Perron-Frobenius operator $K$ is bounded, then the value~\eqref{eq:a_conv} converges to $0$.
And, according to Proposition~\ref{prop:normalized_property}, as $\epsilon$ becomes smaller, the theoretical accuracy is improved and the value~\eqref{eq:a_conv} approaches to $0$ as $T$ becomes large.
On the other hand, as $\epsilon$ becomes smaller, computations can be numerically unstable.
This trade-off is apparent in Figure~\ref{fig:tradeoff}.
\begin{figure}[t]
\begin{center}
\includegraphics[width=.7\linewidth,height=.48\linewidth]{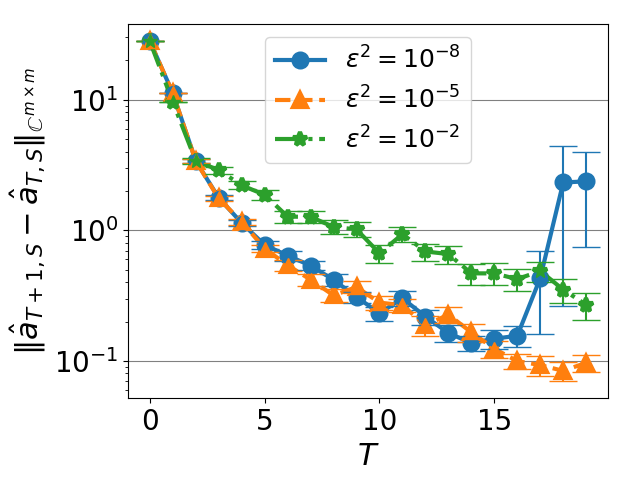}
 \caption{Convergence of $\hat{a}_{T,S}$ along $T$ for different $\epsilon^2$. 
Error bars correspond to the standard errors.
}\label{fig:tradeoff}
\end{center}
\end{figure}
\subsection{Kernel PCA with RKHMs}\label{sec:exp_pca}
{\bf Synthetic Data:\ }
We first randomly generated the following four kinds of sample-sets, each of which has 20 samples composed of three elements in $\mathbb{R}^2$:
\begin{align*}
&([0,0], [0,0], [0,0])+\xi_{t,1},\ 
([1,1], [0,0], [0,0])+\xi_{t,2},\\
&([0,0], [1,1], [0,0])+\xi_{t,3},\ 
([0,0], [0,0], [1,1])+\xi_{t,4},
\end{align*}
where $\xi_{t,j}\in (\mathbb{R}^2)^3$ is random numbers generated from the Gaussian distribution with mean $0$ and standard deviation $0.1$ for $j=1,\ldots,4$ and $t=1,\ldots,20$.
Thus, in this case, the number of elements of each sample $m$ is $3$, and the number of samples $n$ is $80$.
We call the above four sample-sets data1, data2, data3, and data4, respectively.

We applied our kernel PCA with RKHMs to this dataset and computed the matrix-valued coefficients of the first and second principal components (PCs) $\blacket{p_s,\phi(x_t)}_k\ (s=1,2)$ for each sample.
The graphs in Figure~\ref{fig:pca_syn} show the embedding of the components $\blacket{p_s,\phi(x_t)}_k\ (s=1,2)$ with 
the standard PCA in $\mathbb{R}^m$ and the norm in $\mat$ for $m=3$, respectively.
Concerning the former method, since $\blacket{p_s,\phi(x_t)}_k$ is represented as $\blacket{1/\sqrt{\sigma_s}[\mathrm{v}_s,0,\ldots,0],W^*\phi(x_t)}_{(\mat)^n}$, whose rows are $1/\sqrt{\sigma_s}\mathrm{v}_s^*W^*\phi(x_t)$ for the first one and ${0}$ otherwise, we computed the real-valued coefficients of the first PC of $1/\sqrt{\sigma_s}\mathrm{v}_s^*W^*\phi(x_t)$ for $s=1,2$ and $t=1,\ldots,n$ with PCA in $\mathbb{R}^m$. 
And, concerning the latter one, $\Vert \blacket{p_s,\phi(x_t)}_k\Vert_{\mat}$ was computed for $s=1,2$ and $t=1,\ldots,n$.
As can be seen, whereas all the four sample-sets are separated with the PCA in $\mathbb{R}^m$ (left), data2, data3 and data4 are plotted in one cluster with the norm in $\mat$ (right).
This is because the norm of a matrix is invariant with respect to permutations of columns or rows in the matrix.
This implies that, unlike kernel PCA with RKHSs, our kernel PCA with RKHMs can extract not only the similarities between all combinations of pairs of elements but also the invariance with respect to those permutations.

{\bf Human Sckeleton Data:\ }
Next, we applied our kernel PCA with RKHMs to real-world human skeleton data:
SBU-Kinect-Interaction dataset ver.\@ 2.0~\cite{yun12}, 
which includes human skeleton data of 30 points in three dimensions depicting two-person interactions.
We used the middle frame of each video sequence for three kinds of human activities: ``Hug,'' ``Shake hands (S.H.),'' and ``Punch.'' 
The data include 13 pairs of people doing all three activities.
That, $m=3$ and $n=39$ in this case.

The results are shown in Figure~\ref{fig:pca_real}.
Here, for comparison, we also applied the standard kernel PCA with an RKHS associated with the Laplacian kernel on $\mathbb{R}^{90}$.
Our kernel PCA with the RKHM seems to separate the activities more clearly than the kernel PCA with the RKHS.
This would be because, whereas the activities are recognized as three dimensional spatial data in the case of the RKHM, they are recognized as the combination of one dimensional data in the case of the RKHS.
\begin{figure}[t]
\begin{center}
\includegraphics[width=.48\linewidth,height=.365\linewidth]{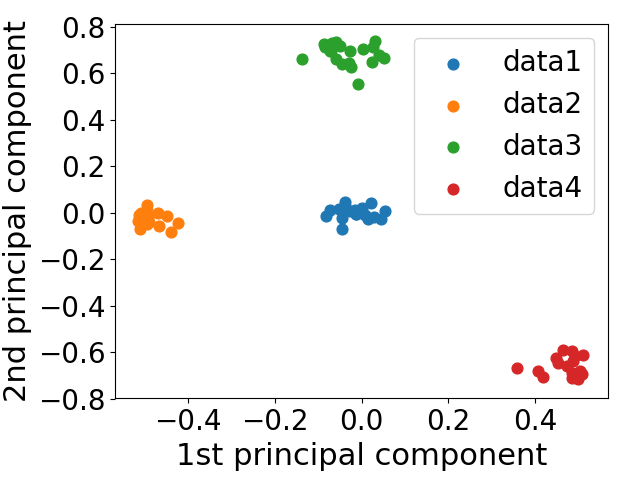} 
\includegraphics[width=.48\linewidth,height=.365\linewidth]{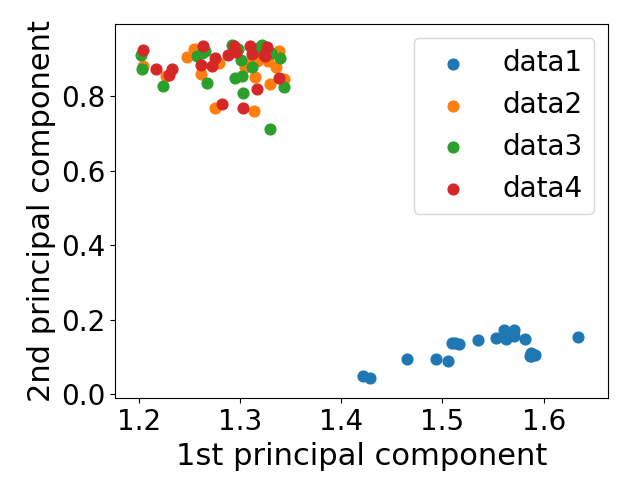} 
 \caption{1st and 2nd PCs of our kernel PCA with RKHMs for synthetic data through PCA in $\mathbb{R}^m$\! (left) and norm in $\mat$\! (right).}\label{fig:pca_syn}
\end{center}
\end{figure}
\begin{figure}[t]
\begin{center}
\includegraphics[width=.48\linewidth,height=.365\linewidth]{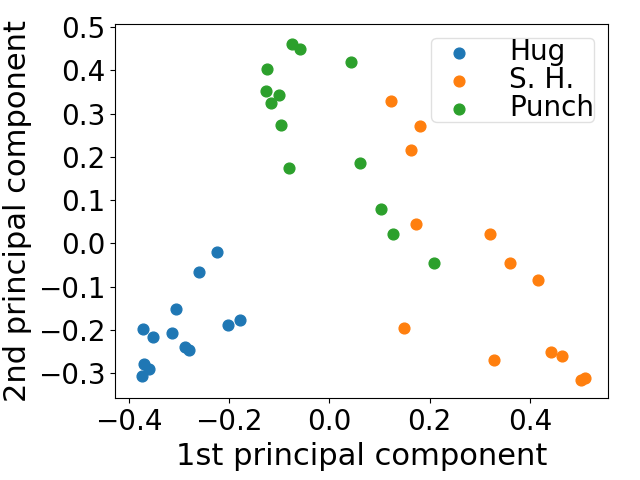} 
\includegraphics[width=.48\linewidth,height=.365\linewidth]{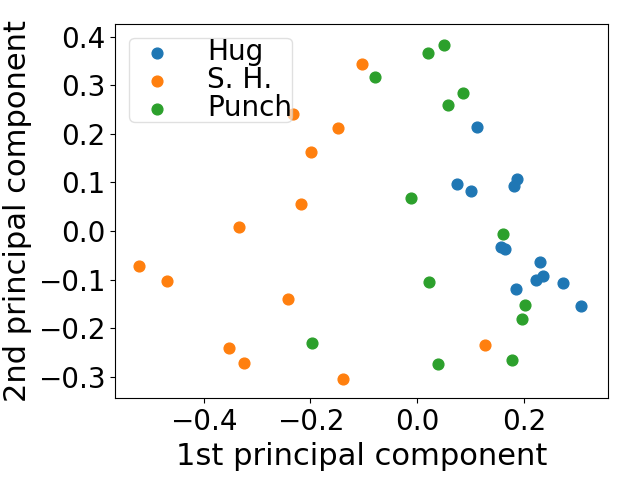} 
 \caption{1st and 2nd PCs of our kernel PCA with RKHMs (left) and of the standard kernel PCA with RKHSs (right) for real-world human skeleton data.}\label{fig:pca_real}
\end{center}
\end{figure}
\begin{figure}[t]
\begin{minipage}{0.475\linewidth}
\begin{center}
\includegraphics[width=.98\linewidth,height=.67\linewidth]{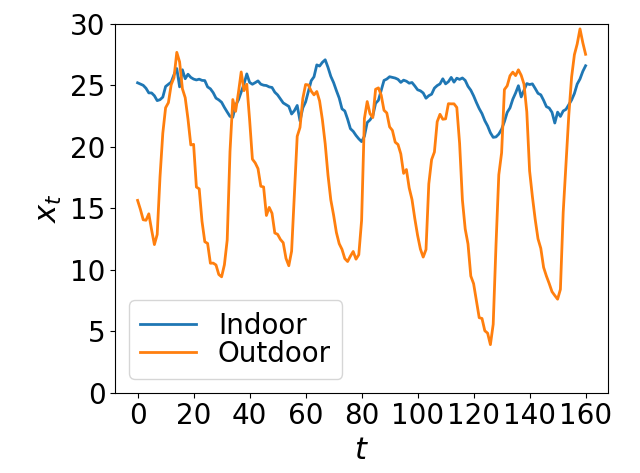}
 \caption{Example of a pair of indoor and outdoor temperature sequences.}\label{fig:data}
\end{center}
\end{minipage}\quad
\begin{minipage}{0.475\linewidth}
\begin{center}
\includegraphics[width=\linewidth,height=.67\linewidth]{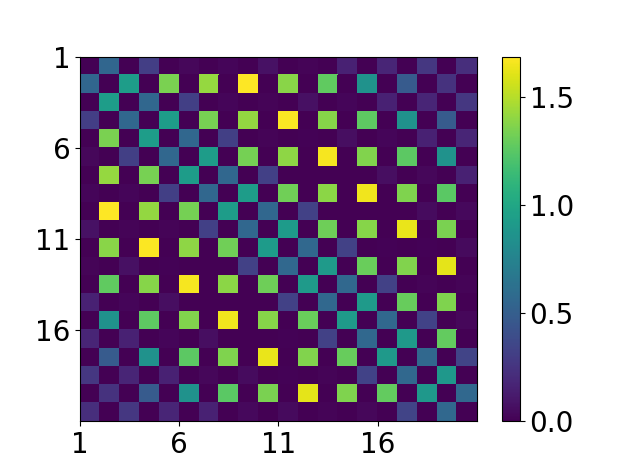}
 \caption{Color map representing the values of elements of estimated matrix $c_{\opn{inv}}$.}\label{fig:causal}
\end{center}
\end{minipage}%
\end{figure}
\subsection{Analysis of dynamical systems with RKHMs}\label{sec:exp_PF}
Finally, we applied our method for extracting time-invariant relations in interacting dynamical systems by using real-world time-series data from a dataset with cause-effect pairs~\cite{mooij16}.
Two time-series data about indoor and outdoor temperatures (pairs0048) at time $0\sim 160$ were used.
The original sequences are shown in Figure~\ref{fig:data}.
We can see that the indoor temperature becomes high $2\sim 5$ time steps after the outdoor one does.
For the mathematical treatments mentioned in Section~\ref{sec:PF_rkhm}
, we added small noises to the original data by using random numbers generated by the Gaussian distribution with mean $0$ and standard deviation $0.2$. 
After that, we normalized the data so that the mean is $0$ and the standard deviations is $1$.
We set $x_t=[y_t,z_t,y_{t+1},z_{t+1},\ldots, y_{t+9},z_{t+9}]\in\mathbb{R}^{20}$, where $y_t$ and $z_t$ are the indoor and outdoor temperatures at $t$, respectively.
Then, we computed the estimation of the Perron-Frobenius operator $K$ and the time invariant term $c_{\opn{inv}}$ in Eq.~\eqref{eq:eigsum}.
In this case, $m=20$, and we set $T$ as $150$.

The values of elements of matrix $c_{\opn{inv}}$ are shown in Figure~\ref{fig:causal}.
We also add figures in Appendix~\ref{ap:figure}.
A point at $(i,j)$ represents the value of $(i,j)$-element of $c_{\opn{inv}}$.
We can see the values of $(2i-1,2j)$ and $(2j,2i-1)$ for $i=j+2\sim j+5$ 
are large, which implies that $y_{j+2}\sim y_{j+5}$ and $z_{i}$ are similar.

\section{Conclusions}
\label{sec:concl}
In this paper, we proposed a new data analysis framework with RKHM.
We showed the theoretical validity for constructing orthonormal systems in RKHMs. 
Then, we derived concrete procedures for orthonormalization in RKHMs,
and applied those to generalize with RKHM kernel principal component analysis and the analysis of dynamical systems with Perron-Frobenius operators.
This enables us to access and extract the rich information of structured data by introducing a positive definite kernel that describes similarities between all combinations of pairs of elements of data.
Numerical results with synthetic and real-world data also show the advantage of our methods with RKHMs.


\newpage
\quad
\newpage
\appendix
\allowdisplaybreaks[4]
We explain the notations used in this paper in Section~\ref{ap:notation}.
Then, we briefly review existing methods in Sections~\ref{ap:review_rkhs}, \ref{ap:vv-RKHS_review}, \ref{ap:pca_review}, and \ref{sec:PF_review} and detailed statements and definitions about RKHMs with their proofs in Section~\ref{ap:rkhm}.
In addition, we explain in detail the mathematical treatments for defining Perron-Frobenius operators discussed in Section~\ref{sec:PF_rkhm} 
and show figures about the numerical results in Section~\ref{ap:figure}.
Finally, we provide proofs of theorems, propositions, corollaries, and lemmas in Section~\ref{ap:proof} and pseudo-codes in Section~\ref{ap:psudocode}.

\section{Notations}\label{ap:notation}
In this section, we describe notations used in this paper.
Small letters denote $\alg$-valued coefficients (often by $a,b,c,d$) or vectors in $\modu$ (often by $p,q,u,v,w$).
Small Greek letters denote $\mathbb{C}$-valued coefficients.
Calligraphic capital letters denote sets.
Bold capital letters denote finite dimentional $\alg$-linear map and bold small letters denote vectors in $\alg^n$ for $n\in\mathbb{N}$ (a finite dimentional Hilbert $C^*$-module).
Small Roman letters denote vectors in $\mathbb{C}^n$ (a finite dimensional vector space).
Also, we use $\sim$ for objects related to RKHSs.

The typical notations in this paper are listed in Table~\ref{tab1} at the last page of this document.

\begin{table}[t]
\caption{Notation table}\label{tab1}\vspace{.3cm}
\renewcommand{\arraystretch}{1.25}
 \begin{tabularx}{\linewidth}{|c|X|}
\hline
$\mat$  & A set of all complex-valued $m\times m$ matrix\\
$\alg$ & A $C^*$-algebra\\
$\Vert\cdot\Vert_{\alg}$ & The norm in $\alg$ (For $\alg=\mat$, $\Vert c\Vert_{\mat}:=\sup_{\Vert\mathrm{d}\Vert_2=1}\Vert c\mathrm{d}\Vert_2$)\\
$\modu$ & A (right) $\alg$-module\\
$\vert\cdot\vert$ & The $\alg$-valued absolute value in $\modu$ defined as a positive element $\vert u\vert$ such that $\vert u\vert^2=\blacket{u,u}$\\
$\Vert\cdot\Vert$ & The norm in $\modu$ defined as $\Vert u\Vert:=\Vert\vert \blacket{u,u}\vert\Vert_{\alg}$\\
$\mcl{X}$ & A topological space\\
$m$ & A natural number that represents the number of elements of data\\
$n$, $T$ & Natural numbers that represent the amount of observed data used for the estimation ($T$ for time-series data)\\
$\mcl{Y}$ & A set of observed data $\mcl{Y}:=\{x_0,x_1,\ldots\}\subseteq \mcl{X}^m$ (a countable and dense subset of $\mcl{X}^m$ whose elements are completely different)\\
$k$ & An $\alg$-valued positive definite kernel\\
$\phi$ & The feature map endowed with $k$\\
$\modu_k$ & The RKHM associated with $k$\\
$\alg^{\mcl{Y}}$ & The set of all functions from $\mcl{Y}$ to $\alg$\\
$\vert\cdot\vert_k$ & The $\alg$-valued absolute value in $\modu_k$ \\
$\Vert\cdot\Vert_k$ & The norm in $\modu_k$ \\
$\tilde{k}$ & A complex-valued positive definite kernel\\
$\tilde{\phi}$ & The feature map endowed with $\tilde{k}$\\
$\hil_{\tilde{k}}$ & The RKHS associated with $\tilde{k}$\\
$\epsilon$ & The parameter that determines the theoretical accuracy and numerical stability\\
$W$ & The linear operator from $\alg^n$ to $\modu_k$ composed of observed data  $\phi(x_1),\ldots,\phi(x_n)$\\
$\mathbf{G}$ & A Gram matrix\\
$\{q_t\}_{t=1}^{\infty}$ & An orthonormal system in $\modu$ or $\modu_k$\\
$Q$ & The linear operator from $\alg^n$ to $\modu_k$ composed of $q_1,\ldots,q_n$\\
$\mathbf{R}$ & The $n\times n$ $\alg$-valued matrix that satisfies $\Vert W-Q\mathbf{R}\Vert\le\epsilon$\\
$\mathbf{R}_{\opn{inv}}$ & The $n\times n$ $\alg$-valued matrix that satisfies $Q=W\mathbf{R}_{\opn{inv}}$\\
$p_s$ & The $s$-th principal axis generated by kernel PCA with an RKHM\\ 
$K$ & A Perron-Frobenius operator on $\modu_k$\\
$\mathbf{K}_T$ & The estimation of $K$ with observed data $\{x_0,\ldots,x_{T-1}\}$\\
$\hat{a}_{T,S}$ & The abnormality at $S$ computed with $\mathbf{K}_T$\\
\hline
 \end{tabularx}
\renewcommand{\arraystretch}{1} 
\end{table}

\section{RKHS}\label{ap:review_rkhs}
In this section, we review the theory of RKHSs.
RKHSs are Hilbert spaces to extract nonlinearity or higher-order moments of data~\cite{scholkopf01,saitoh16}.

We begin with a positive difinite kernel.
Let $\mcl{Y}$ be a non-empty set for data, and $\tilde{k}$ be a positive definite kernel, which is defined as follows:
\begin{defin}\label{def:pdk_rkhs}
A map $\tilde{k}:\mcl{Y}\times \mcl{Y}\to\mathbb{C}$ is called a {\em positive definite kernel} if it satisfies the following conditions:

1. $\tilde{k}(x,y)=\overline{\tilde{k}(y,x)}$\; for $x,y\in\mcl{Y}$\\
2. $\sum_{s,t=1}^{n}\overline{c_s}c_t\tilde{k}(x_s,x_t)\ge 0$\; for $n\in\mathbb{N}$, $c_1,\ldots,c_{n}\in\mathbb{C}$, $x_1,\ldots,x_{n}\in\mcl{Y}$.
\end{defin}
Let $\tilde{\phi}:\mcl{Y}\to\mathbb{C}^{\mcl{Y}}$ be a map defined as $\tilde{\phi}(x)=\tilde{k}(\cdot,x)$.
With $\tilde{\phi}$, the following space as the subset of $\mathbb{C}^{\mcl{Y}}$ is constructed: 
\begin{equation*}
\hil_{\tilde{k},0}:=\bigg\{\sum_{t=1}^{n}c_t\tilde{\phi}(x_t)\bigg|\ n\in\mathbb{N},\ c_t\in\mathbb{C},\ x_t\in\mcl{Y}\bigg\}.
\end{equation*}
Then, a map $\blacket{\cdot,\cdot}_{\tilde{k}}:\hil_{\tilde{k},0}\times \hil_{\tilde{k},0}\to\mathbb{C}$ is defined as follows:
\begin{equation*}
\Bblacket{\sum_{s=1}^{n}c_s\tilde{\phi}(x_s),\sum_{t=1}^{l}d_t\tilde{\phi}(y_t)}_{\tilde{k}}:=\sum_{s=1}^{n}\sum_{t=1}^{l}\overline{c_s}d_t\tilde{k}(x_s,y_t).
\end{equation*}
By the properties in Definition~\ref{def:pdk_rkhs} of $\tilde{k}$, $\blacket{\cdot,\cdot}_{\tilde{k}}$ is well-defined, satisfies the axiom of inner products, and has the reproducing property, that is,
\begin{equation*}
\sblacket{\tilde{\phi}(x),v}_{\tilde{k}}=v(x),
\end{equation*}
for $v\in\hil_{\tilde{k},0}$ and $x\in\mcl{Y}$.

The completion of $\hil_{\tilde{k},0}$ 
is called {\em RKHS} associated with $\tilde{k}$ and denoted as $\hil_{\tilde{k}}$.
It can be shown that $\blacket{\cdot,\cdot}_{\tilde{k}}$ is extended continuously to $\hil_{\tilde{k}}$ and the map $\hil_{\tilde{k}}\ni v\mapsto (x\mapsto\sblacket{\tilde{\phi}(x),v}_{\tilde{k}})\in\mathbb{C}^{\mcl{Y}}$ is injective.
Thus, $\hil_{\tilde{k}}$ is regarded to be the subset of $\mathbb{C}^{\mcl{Y}}$ and has the reproducing property.
Also, by the Moore-Aronszajn theorem, $\hil_{\tilde{k}}$ is determined uniquely.

$\tilde{\phi}$ maps data into $\hil_{\tilde{k}}$, whose dimension is generally higher (often infinite dimensional) than that of $\mcl{Y}$, and is called the {\em feature map}.
Since the dimension of $\hil_{\tilde{k}}$ is higher than that of $\mcl{Y}$, complicated behaviors of data in $\mcl{Y}$ are often transformed into simple ones in $\hil_{\tilde{k}}$~\cite{scholkopf01}.

\section{vv-RKHS}\label{ap:vv-RKHS_review}
In this section, we review the theory of vv-RKHSs.

Similar to the case of RKHSs, we begin with a positive definite kernel.
Let $\mcl{Y}$ be a non-empty set for data, $\mcl{W}$ be a Hilbert space equipped with an inner product $\blacket{\cdot,\cdot}_{\mcl{W}}$,
and $\mcl{L}(\mcl{W})$ be the space of bounded linear operators on $\mcl{W}$.
In addition, $k$ be an operator valued positive definite kernel, which is defined as follows:
\begin{defin}\label{def:pdk_vv-rkhs}
A map map $k:\mcl{Y}\times \mcl{Y}\to\mcl{L}(\mcl{W})$ is called an {\em operator valued positive definite kernel} if it satisfies the following conditions:\vspace{.2cm}\\
1. $k(x,y)=k(y,x)^*$\; for $x,y\in\mcl{Y}$\\
2. $\sum_{s,t=1}^{n}\blacket{v_s,k(x_s,x_t)v_t}_{\mcl{W}}\ge 0$\; for $n\in\mathbb{N}$, $v_1,\ldots,v_{n}\in\mcl{W}$, $x_1,\ldots,x_{n}\in\mcl{Y}$.
\end{defin}
For $v\in\mcl{W}$, let $\phi_v:\mcl{Y}\to\mcl{W}^{\mcl{Y}}$ be a map defined as $\phi_v(x)=k(\cdot,x)v$.
With $\phi$, the following space as the subset of $\mcl{W}^{\mcl{Y}}$ is constructed: 
\begin{equation*}
\hil_{k,0}^{\opn{v}}:=\bigg\{\sum_{t=0}^n\phi_{v_t}(x_t)\bigg| n\in\mathbb{N},\ v_t\in\mcl{W},\ x_t\in\mcl{Y},\bigg\}.
\end{equation*}
Then, a map $\blacket{\cdot,\cdot}_{k,{\opn{v}}}:\hil_{k,0}^{\opn{v}}\times \hil_{k,0}^{\opn{v}}\to\mathbb{C}$ is defined as follows:
\begin{align*}
&\Bblacket{\sum_{s=1}^n\phi_{v_t}(x_t),\sum_{t=0}^l\phi_{w_t}(y_t)}_{k,{\opn{v}}}\\
&\qquad:=\sum_{s=1}^n\sum_{t=1}^l\blacket{v_s,k(x_s,y_t)w_t}_{\mcl{W}}.
\end{align*}
By the properties in Definition~\ref{def:pdk_vv-rkhs} of $k$, $\blacket{\cdot,\cdot}_{k,{\opn{v}}}$ is well-defined, satisfies the axiom of inner products and has the reproducing property, that is,
\begin{equation*}
\blacket{\phi_v(x),u}_{k,{\opn{v}}}=\blacket{v,u(x)}_{\mcl{W}},
\end{equation*}
for $u\in\hil_{k,0}^{{\opn{v}}}$ and $x\in\mcl{Y}$.

The completion of $\hil_{k,0}^{\opn{v}}$ 
is called {\em vv-RKHS} associated with $k$ and denoted as $\hil_{k,{\opn{v}}}$.
Note that since the inner product in $\hil_k^{\opn{v}}$ is defined with the complex-valued inner product in $\mcl{W}$, it is complex-valued.

\section{Statements and definitions about RKHMs and their proofs}\label{ap:rkhm}
In this section, we provide the precise statements and definitions about RKHMs introduced in Section~\ref{sec:bg} and show their proofs.
\begin{defin}[$C^*$-algebra]~\label{def:c*_algebra}
A set $\mcl{A}$ is called a {\em $C^*$-algebra} if it satisfies the following conditions:

1. $\mcl{A}$ is an algebra over $\mathbb{C}$, and there exists a bijection $(\cdot)^*:\mcl{A}\to\mcl{A}$ that satisfies the following conditions for $\lambda,\mu\in\mathbb{C}$ and $c,d\in\mcl{A}$:

\leftskip=10pt
$\bullet$ $(\lambda c+\mu d)^*=\overline{\lambda}c^*+\overline{\mu}d^*$\\
\ $\bullet$ $(cd)^*=d^*c^*$\\
\ $\bullet$ $(c^*)^*=c$

\leftskip=0pt
2. $\mcl{A}$ is a norm space with $\Vert\cdot\Vert_{\alg}$, and for $c,d\in\mcl{A}$, $\Vert cd\Vert_{\alg}\le\Vert c\Vert_{\alg}\Vert d\Vert_{\alg}$ holds.
In addition, $\mcl{A}$ is complete with respect to $\Vert\cdot\Vert_{\alg}$.

3. For $c\in\mcl{A}$, $\Vert c^*c\Vert_{\alg}=\Vert c\Vert_{\alg}^2$ holds.
\end{defin}

\begin{defin}[Positive]~\label{def:positive}
$c\in\mcl{A}$ is called {\em positive} if there exists $d\in\mcl{A}$ such that $c=d^*d$ holds.
For a positive element $c\in\mcl{A}$, we denote $c\ge 0$.
\end{defin}

\begin{defin}[(Right) multiplication]\label{def:multiplication}
Let $\modu$ be an abelian group with operation $+$.
For $c,d\in\alg$ and $u,v\in\modu$, if an operation $\cdot:\modu\times\alg\to\modu$ satisfies

1. $(u+v)\cdot c=u\cdot c+v\cdot c$\\
2. $u\cdot (c+d)=u\cdot c+u\cdot d$\\
3. $u\cdot (cd)=(u\cdot d)\cdot c$\\
4. $u\cdot 1_{\alg}=u$,

where $1_{\alg}$ is the multiplicative identity of $\alg$,
then, $\cdot$ is called (right) {\em $\alg$-multiplication}.
The multiplication $u\cdot c$ is usually denoted as $uc$.
\end{defin}

\begin{remark}
For practical applications of the theory of RKHSs, considering column vectors rather than row vectors is standard for representing coefficients.
Column vectors act on the right.
Therefore, we consider right multiplications for making algorithms with RKHMs  (with $\alg$-multiplications) compatible with those with RKHSs.
\end{remark}

\begin{defin}[$C^*$-module]\label{def:c*module}
Let $\modu$ be an abelian group with operation $+$.
If $\modu$ has the structure of a (right) $\alg$-multiplication, $\modu$ is called a (right) {\em $C^*$-module} over $\alg$.
\end{defin}
 
\begin{defin}[Hilbert $C^*$-module]\label{def:hil_c*module}
Let $\modu$ be a (right) $C^*$-module over $\alg$ which is equipped with an $\alg$-valued inner product defined in Definition~\ref{def:innerproduct}.
If $\modu$ is complete with respect to the norm $\Vert \cdot\Vert$, it is called a {\em Hilbert $C^*$-module} over $\alg$.
\end{defin}

\begin{lemma}[Cauchy-Schwarz inequality~\cite{lance95}]\label{lem:c-s}
For $u,v\in\modu$, the following inequality holds:
\begin{equation*}
\vert\blacket{u,v}\vert^2\le\Vert u\Vert^2\blacket{v,v}.
\end{equation*}
\end{lemma}

\begin{remark}\label{rem:c-s}
The proof of Cauchy-Shwarz inequality only requires properties 1 and 2 about $\alg$-valued inner products in Definition~\ref{def:innerproduct}.
\end{remark}

\begin{lemma}\label{lem:pdk_equiv}
Let $\mcl{W}$ be a Hilbert space equipped with an inner product $\blacket{\cdot,\cdot}_{\mcl{W}}$,
and $\mcl{L}(\mcl{W})$ be the space of bounded linear operators on $\mcl{W}$.
If $\alg=\mcl{L}(\mcl{W})$, then, the $\alg$-valued positive definite kernel defined in Definition~\ref{def:pdk_rkhm} is equivalent to the operator valued positive definite kernel defined in Definition~\ref{def:pdk_vv-rkhs}.
\end{lemma}
\begin{proof}
Let $k$ be an $\alg$-valued positive definite kernel defined in Definition~\ref{def:pdk_rkhm}.
Let $v\in\mcl{W}$.
For $n\in\mathbb{N}$, $v_1,\ldots,v_n\in\mcl{W}$, let $c_t\in\mcl{L}(\mcl{W})$ be defined as $c_tu:=\blacket{v,u}_{\mcl{W}}/\blacket{v,v}_{\mcl{W}}v_t$ for $u\in\mcl{W}$.  
Since $v_t=c_tv$ holds, the following equalities are derived for $x_1,\ldots,x_n\in\mcl{Y}$:
\begin{align*}
\sum_{s,t=1}^n\blacket{v_t,k(x_t,x_s)v_s}_{\mcl{W}}
&=\sum_{s,t=1}^n\blacket{c_tv,k(x_t,x_s)c_sv}_{\mcl{W}}\\
&=\Bblacket{v,\sum_{s,t=1}^nc_t^*k(x_t,x_s)c_sv}_{\mcl{W}}.
\end{align*}
By the positivity of $\sum_{s,t=1}^nc_t^*k(x_t,x_s)c_s$, $\sblacket{v,\sum_{s,t=1}^nc_t^*k(x_t,x_s)c_sv}_{\mcl{W}}\ge 0$ holds, which implies $k$ is an operator valued positive definite kernel defined in Definition~\ref{def:pdk_vv-rkhs}.

On the other hand, let $k$ be an operator valued positive definite kernel defined in Definition~\ref{def:pdk_vv-rkhs}.
Let $v\in\mcl{W}$.
For $n\in\mathbb{N}$, $c_1,\ldots,c_n\in\alg$ and $x_1,\ldots,x_n\in\mcl{Y}$, the following equality is derived:
\begin{align*}
\Bblacket{v,\sum_{s,t=1}^nc_t^*k(x_t,x_s)c_sv}_{\mcl{W}}\!\!\!\!
=\sum_{s,t=1}^n\blacket{c_tv,k(x_t,x_s)c_sv}_{\mcl{W}}.
\end{align*}
By Definition~\ref{def:pdk_vv-rkhs}, $\sum_{s,t=1}^n\blacket{c_tv,k(x_t,x_s)c_sv}_{\mcl{W}}\ge 0$ holds, which implies $k$ is an $\alg$-valued positive definite kernel defined in Definition~\ref{def:pdk_rkhm}.
\end{proof}

\begin{prop}\label{prop:inner_product}
$\blacket{\cdot,\cdot}_k$ is an $\alg$-valued inner product.
\end{prop}
\begin{proof}
Property 1 in Definition~\ref{def:innerproduct} is followed by the definition of $\blacket{\cdot,\cdot}_k$.
The following equality for $v:=\sum_{t=0}^n\phi(x_t)c_t$ and $w:=\sum_{s=0}^l\phi(y_s)d_s$ implies $\blacket{\cdot,\cdot}_k$ satisfies property 2:
\begin{align*}
\blacket{v,w}_k&=\Bblacket{\sum_{t=0}^n\phi(x_t)c_t,\sum_{s=0}^l\phi(y_s)d_s}_k\\
&=\sum_{t=0}^n\sum_{s=1}^lc_t^*k(x_t,y_s)d_s\\
&=\bigg(\sum_{t=0}^n\sum_{s=0}^ld_s^*k(y_s,x_t)c_t\bigg)^*\\
&=\Bblacket{\sum_{s=0}^l\phi(y_s)d_s,\sum_{t=0}^n\phi(x_t)c_t}_k^*\\
&=\blacket{w,v}_k^*.
\end{align*}
Concerning property 3, $\blacket{u,u}_k\ge 0$ for $u\in\modu_k$ holds by the positive definiteness of $k$ (Definition~\ref{def:pdk_rkhm}.2).
In addition, by Cauchy-Schwarz inequality (Lemma~\ref{lem:c-s} and Remark~\ref{rem:c-s}), the following inequality holds for $x\in\mcl{Y}$: 
\begin{equation*}
\vert v(x)\vert^2_{\alg}=\vert\blacket{\phi(x),v}_k\vert^2_{\alg}\le\Vert k(x,x)\Vert_{\alg}\blacket{v,v}_k.
\end{equation*}
Thus, if $\blacket{v,v}_k=0$, then $\vert v(x)\vert_{\alg}=0$ for all $x\in\mcl{Y}$, which implies $v=0$.
\end{proof}

\begin{prop}\label{prop:reproducing}
$\blacket{\cdot,\cdot}_{k}$ is extended continuously to $\modu_k$ and the map $\modu_k\ni v\mapsto (x\mapsto\blacket{\phi(x),v}_{k})\in\alg^{\mcl{Y}}$ is injective.
Thus, $\modu_k$ is regarded to be the subset of $\alg^{\mcl{Y}}$ and has the reproducing property.
\end{prop}
\begin{proof}
(Existence) For $v,w\in\mcl{M}_k$, there exist $v_t,w_t\in\mcl{M}_{k,0}\ (t=1,2,\ldots)$ such that $v=\lim_{t\to\infty}v_t$ and $w=\lim_{t\to\infty}w_t$.
By Cauchy-Schwarz inequality (Lemma~\ref{lem:c-s}), the following inequalities hold:
\begin{align*}
&\Vert \blacket{v_t,w_t}_k-\blacket{v_s,w_s}_k\Vert_{\alg}\\
&\quad\le\Vert \blacket{v_t,w_t-w_s}_k\Vert_{\alg}+\Vert \blacket{v_t-v_s,w_s}_k\Vert_{\alg}\\
&\quad\le \Vert v_t\Vert_k\Vert w_t-w_s\Vert_k+\Vert v_t-v_s\Vert_k\Vert w_s\Vert_k\\
&\quad\to 0\ (t,s\to\infty),
\end{align*}
which implies $\{\blacket{v_t,w_t}_k\}_{t=1}^{\infty}$ is a Cauchy sequence in $\mcl{A}$.
By the completeness of $\mcl{A}$, there exists a limit $\lim_{t\to\infty}\blacket{v_t,w_t}_k$.

(Well-definedness) Assume there exist $v'_t,w'_t\in\mcl{M}_{k,0}\ (t=1,2,\ldots)$ such that $v=\lim_{t\to\infty}v_t=\lim_{t\to\infty}v'_t$ and $w=\lim_{t\to\infty}w_t=\lim_{t\to\infty}w'_t$.
By Cauchy-Schwarz inequality (Lemma~\ref{lem:c-s}), $\Vert\blacket{v_t,w_t}_k-\blacket{v'_t,w'_t}_k\Vert_{\alg}\le\Vert v_t\Vert_k\Vert w_t-w'_t\Vert_k+\Vert v_t-v'_t\Vert_k\Vert w'_t\Vert_k\to 0\ (t\to\infty)$ holds, which implies $\lim_{t\to\infty}\blacket{v_t,w_t}_k=\lim_{t\to\infty}\blacket{v'_t,w'_t}_k$.

(Injectivity) For $v,w\in\mcl{M}_k$, we assume $\blacket{\phi(x),v}_k=\blacket{\phi(x),w}_k$ for $x\in\mcl{Y}$.
By the linearity of $\blacket{\cdot,\cdot}_k$, $\blacket{u,v}_k=\blacket{u,v}_k$ holds for $u\in\mcl{M}_{k,0}$.
For $u\in\modu_k$, there exist $u_t\in\modu_{k,0}\ (t=1,2,\ldots)$ such that $u=\lim_{t\to\infty}u_t$.
Therefore, $\blacket{u,v-w}_k=\lim_{t\to\infty}\blacket{u_t,v-w}_k=0$.
As a result, $\blacket{v-w,v-w}_k=0$ holds by setting $u=v-w$, which implies $v=w$.
\end{proof}

\begin{prop}\label{prop:rkhm_unique}
Assume a Hilbert $C^*$-module $\mcl{M}$ over $\mcl{A}$ and a map $\psi:\mcl{Y}\to\mcl{M}$ satisfy the following conditions:\vspace{.2cm}\\
1. $^{\forall} x,y\in\mcl{Y}$, $\blacket{\psi(x),\psi(y)}_{\mcl{M}}=k(x,y)$\vspace{.1cm}\\
2. $\overline{\{\sum_{t=0}^n\psi(x_t)c_t\mid\ x_t\in\mcl{Y},\ c_t\in\mcl{A}\}}=\mcl{M}$\vspace{.2cm}\\
Then, there exists a unique $\mcl{A}$-linear bijection map $\Psi:\mcl{M}_k\to\mcl{M}$ that preserves the inner product and satisfies the following commutative diagram:
\begin{equation}
 \xymatrix{
\mcl{M}_k \ar[rr]^{\Psi}&  &\mcl{M}\\
&\mcl{Y} \ar[lu]^{\phi} \ar[ru]_{\psi}\ar@{}[u]|{\circlearrowright} &
}
\end{equation}
\end{prop}
\begin{proof}
We define $\Psi:\mcl{M}_{k,0}\to\mcl{M}$ as an $\mcl{A}$-linear map that satisfies $\Psi(\phi(x))=\psi(x)$.
We show $\Psi$ can be extended to a unique $\mcl{A}$-linear bijection map on $\modu_k$ , which preserves the inner product.

(Uniqueness) The uniqueness follows by the definition of $\Psi$.

(Inner product preservation) For $x,y\in\mcl{Y}$, equalities  $\blacket{\Psi(\phi(x)),\Psi(\phi(y))}_k=\blacket{\psi(x),\psi(y)}_{\mcl{M}}=k(x,y)=\blacket{\phi(x),\phi(y)}_k$ hold.
Since $\Psi$ is $\alg$-linear, $\Psi$ preserves the inner products between arbitrary $u,v\in\modu_{k,0}$

(Well-definedness) Since $\Phi$ preserves the inner product, if  $\{v_t\}_{t=1}^{\infty}\subseteq\mcl{M}_k$ is a Cauchy sequence, $\{\Psi(v_t)\}_{t=1}^{\infty}\subseteq\mcl{M}$ is also a Cauchy sequence.
Therefore, by the completeness of $\mcl{M}$, $\Psi$ also preserves the inner product in $\mcl{M}_k$, and for $v\in\mcl{M}_k$, $\Vert\Psi(v)\Vert_{\mcl{M}}=\Vert v\Vert_k$ holds.
As a result, for $v\in\mcl{M}_k$, if $v=0$, $\Vert \Psi(v)\Vert_{\mcl{M}}=\Vert v\Vert_k=0$ holds.
This implies $\Psi(v)=0$.

(Injectivity) 
For $v,w\in\mcl{M}_k$, if $\Psi(v)=\Psi(w)$, then 
$0=\Vert\Psi(v)-\Psi(w)\Vert_{\mcl{M}}=\Vert v-w\Vert_k$ holds since $\Psi$ preserves the inner product, which implies $v=w$.

(Surjectivity) The surjectivity follows directly by the condition $\overline{\{\sum_{t=0}^n\psi(x_t)c_t\mid\ x_t\in\mcl{Y},\ c_t\in\mcl{A}\}}=\mcl{M}$.
\end{proof}

\begin{lemma}\label{lem:A^m}
$\alg^n$ is a Hilbert $C^*$-module equipped with an $\alg$-valued inner product $\blacket{\cdot,\cdot}_{\alg^n}$ defined as  $\blacket{\mathbf{u},\mathbf{v}}_{\alg^n}:=\sum_{t=1}^{n}u_t^*v_t$
for $\mathbf{u}=[u_1,\ldots,u_{n}],\mathbf{v}=[v_1,\ldots,v_{n}]\in\alg^n$
where $u_1,\ldots,u_{n},v_1,\ldots,v_{n}\in\alg$.
\end{lemma}
\begin{proof}
Let $k:\alg^n\times\alg^n\to\alg$ be a map defined by $k(\mathbf{u},\mathbf{v})=\sum_{t=1}^{n}u_t^*v_t$ for $\mathbf{u}=[u_1,\ldots,u_{n}],\mathbf{v}=[v_1,\ldots,v_{n}]\in\alg^n$.
Then, $k(\mathbf{u},\mathbf{v})=k(\mathbf{v},\mathbf{u})^*$ holds, and the following equalities hold for $n\in\mathbb{N}$, $c_1,\ldots,c_n\in\alg$ and $\mathbf{u}_1,\ldots,\mathbf{u}_n\in\alg^n$:
\begin{align*}
\sum_{s,t=1}^nc_s^*k(\mathbf{u}_s,\mathbf{u}_t)c_t
&=\sum_{s,t=1}^nc_s^*\sum_{j=1}^m\mathbf{u}_{s,j}^*\mathbf{u}_{t,j}c_t\\
&=\sum_{j=1}^m\bigg(\sum_{t=1}^n\mathbf{u}_{t,j}c_t\bigg)^*\bigg(\sum_{t=1}^n\mathbf{u}_{t,j}c_t\bigg).
\end{align*}
Thus, $\sum_{s,t=1}^nc_s^*k(\mathbf{u}_s,\mathbf{u}_t)c_t$ is positive and $k$ is an $\alg$-valued positive definite kernel.
As a result, $\alg^n$ is the RKHM associated with $k$, which completes the proof of the lemma.
\end{proof}

\section{Kernel PCA with RKHSs}\label{ap:pca_review}
In this section, we briefly review the kernel PCA with RKHSs.

We construct an orthonormal basis of the space spanned by samples that minimizes the reconstruction error.
Let $x_1,\ldots,x_{n}\in\mcl{X}$ be samples.
Let $\tilde{w}_t=\tilde{\phi}(x_t)$ be the samples embedded in an RKHS for $t=1,\ldots,n$, 
$\tilde{W}:=[\tilde{w}_1,\ldots,\tilde{w}_n]$ be the operator that represent the samples, 
and $\tilde{\mathbf{G}}:=[\blacket{\tilde{w}_s,\tilde{w}_t}_{\tilde{k}}]_{s,t}$ be an $n\times n$ $\mathbb{C}$-valued Gram matrix. 
In addition, let $\tilde{\mathbf{G}}=\mathbf{\tilde{V}\tilde{\Sigma} \tilde{V}^*}$ be an eigenvalue decomposition of $\tilde{\mathbf{G}}$,
where $\mathbf{\tilde{\Sigma}}=\opn{diag}\{\tilde{\sigma}_1,\ldots,\tilde{\sigma}_l\}$ and $\tilde{\sigma}_1\ge\ldots\ge\tilde{\sigma}_l> 0$ is the nonzero eigenvalues of $\tilde{\mathbf{G}}$.
The eigenvector corresponding to the largest eigenvalue of operator $\tilde{W}\tilde{W}^*$ is the direction that describes the samples the most, since the squared sum of the samples projected on the space spanned by a normalized vector $\tilde{p}\in\hil_{\tilde{k}}$ is represented as $\sum_{t=1}^n\big\Vert\tilde{p}\bblacket{\tilde{p},\tilde{w}_t}_{\tilde{k}}\big\Vert_{\tilde{k}}^2=\tilde{p}^*\tilde{W}\tilde{W}^*\tilde{p}$.
Therefore, with $\mathbf{\tilde{\Sigma}}$ and $\mathbf{\tilde{V}}$, the eigenvectors of $\tilde{W}\tilde{W}^*$ are explored.
The eigenvalues of $\tilde{W}\tilde{W}^*$ are equal to those of $\tilde{W}^*\tilde{W}$, which is equal to $\tilde{\mathbf{G}}$.
For eigenvalue $\tilde{\sigma}_s$ ($s=1,\ldots,l$), let $\tilde{\mathrm{v}}_s$ be the $s$-th column of $\tilde{\mathbf{V}}$ and
\begin{equation}
\tilde{p}_s:=\tilde{\sigma}_s^{-1/2}\tilde{W}\tilde{\mathrm{v}}_s.\label{eq:p_axis_rkhs} 
\end{equation}
Then, $\tilde{W}\tilde{W}^*\tilde{p}_s=\tilde{\sigma}_s\tilde{p}_s$ holds.
$\tilde{p}_s$ is called the $s$-th {\em principal axis}.
And, for each $w_t$, $\tilde{p}_s\blacket{\tilde{p}_s,w_t}$, the projected vector of $w_t$ onto the space spanned by $\tilde{p}_s$, is called the $s$-th {\em principal component} of $w_t$.

It can be shown that the space spanned by $\{\tilde{p}_t\}_{t=1}^s$ minimizes the reconstruction error, that is, $\{\tilde{p}_t\}_{t=1}^s$ is a solution of the following minimization problem~\citep[Proposition 14.1]{scholkopf01}:
\begin{equation*}
\min_{\tilde{q}_j\in\hil_{\tilde{k}}:\mbox{\small orthonormal}}\sum_{t=1}^n\bigg\Vert \tilde{w}_t-\sum_{j=1}^s\tilde{q}_j\blacket{\tilde{q}_j,\tilde{w}_j}_{\tilde{k}}\bigg\Vert_{\tilde{k}}^2.
\end{equation*}

\section{Analysis of dynamical systems with RKHSs}\label{sec:PF_review}
In this section, we briefly review the existing methods for analyzing time-series data with Perron-Frobenius operators in RKHSs.

First, Perron-Frobenius operators in RKHSs are defined.
Let $\{x_0,x_1,\ldots\}\subseteq \mcl{X}$ be time-series data, which is assumed to be  generated by the following deterministic dynamical system:
\begin{equation}
x_{t+1}=f(x_t),\label{eq:ds_rkhs}
\end{equation}
where $f:\mcl{X}\to\mcl{X}$ is a map.
By embedding $x_t$ and $f(x_t)$ in an RKHS $\hil_{\tilde{k}}$ associated with a positive definite kernel $\tilde{k}$ and the feature map $\tilde{\phi}$, dynmical system~\eqref{eq:ds_rkhs} in $\mcl{X}$ is transformed into that in the RKHS as follows:
\begin{equation*}
\tilde{\phi}(x_{t+1})=\tilde{\phi}(f(x_t)).
\end{equation*}
The Perron-Frobenius operator $\tilde{K}$ in the RKHS is defined as follows~\cite{kawahara16,ishikawa18}:
\begin{equation*}
\tilde{K}\tilde{\phi}(x):=\tilde{\phi}(f(x)).
\end{equation*}
If $\{\tilde{\phi}(x)\mid\ x\in\mcl{X}\}$ is linearly independent, $\tilde{K}$ is well-defined as a linear map in the RKHS.
For example, if $\tilde{k}$ is the Gaussian or Laplacian kernel on $\mcl{X}=\mathbb{R}^d$, $\{\tilde{\phi}(x)\mid\ x\in\mcl{X}\}$ is linearly independent.

For estimating $\tilde{K}$ with observed time-series data  $\{x_0,x_1,\ldots\}$, Krylov subspace methods are applied~\cite{kawahara16,hashimoto19}.
Let $\tilde{W}_T:=[\tilde{\phi}(x_0),\ldots,\tilde{\phi}(x_{T-1})]$ and $\tilde{W}_T=\tilde{Q}_T\tilde{\mathbf{R}}_T$ be the QR decomposition of $\tilde{W}$ in the RKHS.
$\tilde{K}$ is estimated by projecting $\tilde{K}$ into the space spanned by  $\{\tilde{\phi}(x_t)\}_{t=0}^{T-1}$, which is called Krylov subspace.
Since $\tilde{K}\tilde{\phi}(x_t):=\tilde{\phi}(f(x_t))=\tilde{\phi}(x_{t+1})$ holds, the estimation of $\tilde{K}$ can be computed only with observed data as follows:
\begin{align*}
\tilde{\mathbf{K}}_T:&= \tilde{Q}_T^*\tilde{K}\tilde{Q}_T=\tilde{Q}_T^*\tilde{K}\tilde{W}_T\tilde{\mathbf{R}}_T^{-1}\\
&=\tilde{Q}_T^*[\tilde{\phi}(x_1),\ldots,\tilde{\phi}(x_{T})]\tilde{\mathbf{R}}_T^{-1}.
\end{align*}

The estimation of $\tilde{K}$ provides a prediction of data in the RKHS, which can be applied to anomaly detection~\cite{hashimoto19}, information of periodicity of the data~\cite{kawahara16}, and so on.
\citet{hashimoto19} proposed evaluating the prediction error at $S$ with the following value:
\begin{equation}
\hat{\alpha}_{T,S}:=\Vert \tilde{\phi}(x_S)-\tilde{Q}_T\tilde{\mathbf{K}}_T\tilde{Q}_T^*\tilde{\phi}(x_{S-1})\Vert_{\tilde{k}},\label{eq:error_rkhs}
\end{equation}
where $\tilde{Q}_T\tilde{\mathbf{K}}_T\tilde{Q}_T^*\tilde{\phi}(x_{S-1})$ is regarded as a prediction at $S$ with the estimated operator $\tilde{\mathbf{K}}_T$.

For interacting dynamical systems, \citet{fujii19} proposed a modal decomposition of the operator that represents the relation between $k(x_t,x_t)$ and $k(x_{t+1},x_{t+1})$ for observed data $\{x_0,x_1,\ldots\}\subseteq\mcl{X}^m$ and some $\mat$-valued positive definite kernel $k$, under the assumption of the relation is linear.
Whereas, our decomposition~\eqref{eq:eigsum} of $k(x_t,x_t)$ is also valid for the case where the relation between $k(x_t,x_t)$ and $k(x_{t+1},x_{t+1})$ is nonlinear.
In this sense, decomposition~\eqref{eq:eigsum} generalizes the decomposition considered by~\citet{fujii19}.

\section{The mathematical treatments for defining Perron-Frobenius operators in RKHMs}\label{ap:lin_indep}
In this section, we explain in detail the mathematical treatments for defining Perron-Frobenius operators in RKHMs mentioned in Section~\ref{sec:PF_rkhm}.

Unlike in the case of RKHSs, $\{\phi(x)\mid\ x\in\mcl{Y}\}$ is not always linearly independent even for well-behaved kernels like the Gaussian, which makes it difficult to define Perron-Frobenius operators as a $\mat$-linear operator in general $\modu_k$.
Indeed, for $\mcl{Y}\subseteq\mathbb{C}^2$, $x=[1,1]\in\mathbb{C}^2$ and $c=\begin{bmatrix}
			      0&1\\
			      0&-1
			     \end{bmatrix}$
, $\phi(x)c=0$ holds. 
On the other hand, $\phi([f_1(x),f_2(x)])c$ is not always $0$.
This is due to the fact that for some $c,d\in\alg$, $cd$ becomes $0$ even if neither $c$ nor $d$ is $0$.
This situation never happens in RKHSs, whose kernel is complex-valued.

However, if $\mcl{Y}$ has a proper condition, we can extend $K$ to $\modu_{k,0}(\mcl{Y})$ as a $\mat$-linear map.
\begin{prop}\label{prop:linearly_indep}
Assume $x_s,x_t\in\mcl{Y}$ satisfies the following condition:
\begin{equation}\label{eq:completely_different}
x_{s,i}\neq x_{t,j}\;\ \mbox{for}\;\ (s,i)\neq(t,j),\ i,j=1,\ldots,m,
\end{equation}
and $\{\tilde{\phi}(x)\mid\ x\in\mcl{X}\}$ is linearly independent, where $\tilde{\phi}$ is defined in Section~\ref{sec:settings}.
Then, $\{\phi(x)\mid\ x\in\mcl{Y}\}$ is $\mat$-linearly independent.
\end{prop}
\begin{proof}
Let $\sum_{t=0}^{T-1}\phi(x_t)c_t=0$ for $T\in\mathbb{N}$, $x_t\in\mcl{Y}$ and $c_t\in\mat$.
Then $\sum_{t=0}^{T-1}k(y,x_t)c_t=0$ holds for an arbitrary $y\in\mcl{Y}$.
Therefore, $\sum_{t=0}^{T-1}\sum_{l=1}^{m}\tilde{k}(y_i,x_{t,l})(c_t)_{l,j}=0$ holds for arbitrary pairs $(i,j)$ with $i,j=1,\ldots,m$.
Since $y\in\mcl{Y}$ is arbitrary, this implies $\sum_{t=0}^{T-1}\sum_{l=1}^{m}\tilde{\phi}(x_{t,l})(c_t)_{l,j}=0$ holds for $j=1,\ldots,m$.
As a result, $c_t=0$ for $t=0,\ldots,T-1$ since $\{\tilde{\phi}(x)\mid\ x\in\mcl{X}\}$ is linearly independent, which completes the proof of the proposition.
\end{proof}

Therefore, in fact, we consider $\mcl{Y}$ with condition~\eqref{eq:completely_different}, that is, $\mcl{Y}$ is a subset of $\mcl{X}^m$ that is composed of completely different elements in $\mcl{X}$.

Practically, observed data often contains observed noise, which allows us to regard $\{x_0,x_1,\ldots\}$ as data whose elements are completely different.
Otherwise, we artificially add some noise to $\{x_0,x_1,\ldots\}$ and slightly perturb it to meet the condition.
Therefore, the assumption does not cause any difficulties in practical computations.

\section{Figures about the results in Section~\ref{sec:exp_PF}}\label{ap:figure}
Here, we show the figures that illustrate the results about time-invariant relations computed in Section~\ref{sec:exp_PF}.
Values in the first row, which is equal to the first column, of $c_{\opn{inv}}$ represent the time-invariant similarities between $y_t$ and $z_t\sim z_{t+9}$, where $y_t$ and $z_t$ are the indoor and outdoor temperatures at $t$, respectively.
Analogously, values in the second row, which is equal to the second column, of $c_{\opn{inv}}$ represent the time-invariant similarities between $z_t$ and $y_t\sim y_{t+9}$.
Figure~\ref{fig:causal_graph} illustrates the above values with graphs.
The edges in the graphs are directed towards later time steps.
Also, the width of the edges in the graphs are proportional to the corresponding values in $c_{\opn{inv}}$.
We can see $y_{j+2}\sim y_{j+5}$ and $z_{i}$ are similar, which implies the indoor temperature becomes high $2\sim 5$ time steps after the outdoor one does.

\begin{figure}[t]
\begin{center}
\subfigure[\small between $y_t$ and $z_t\sim z_{t+9}$]{
\includegraphics[viewport=200 100 400 300, width=.48\linewidth,clip]{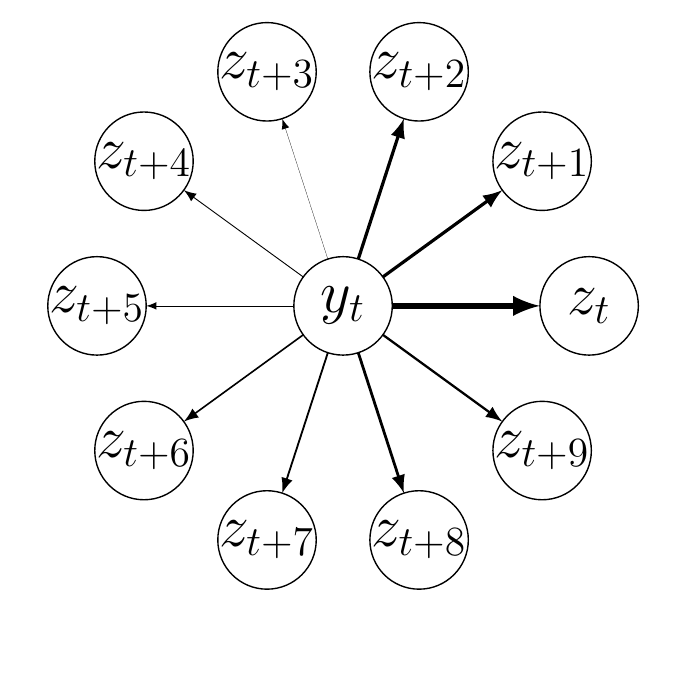}}
\subfigure[\small between $z_t$ and $y_t\sim y_{t+9}$]{
\includegraphics[viewport=200 100 400 300, width=.48\linewidth,clip]{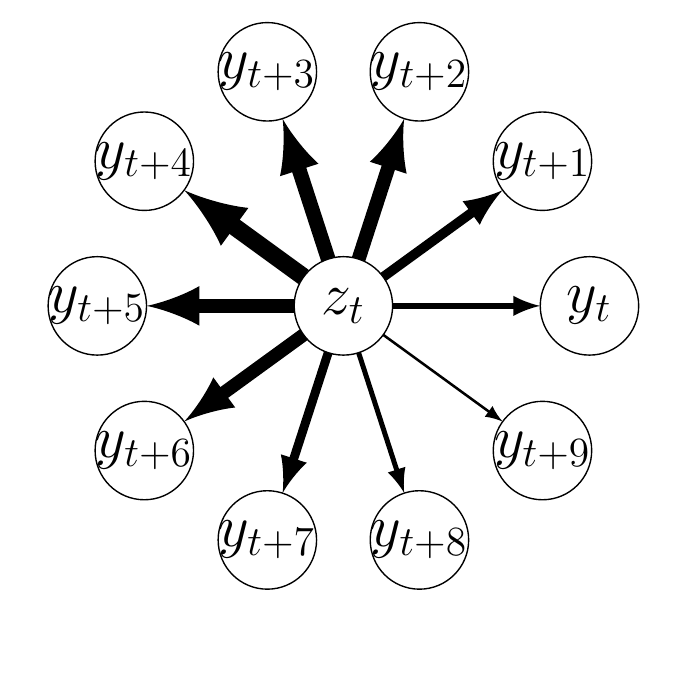}}
 \caption{Graphs illustrating similarities between the indoor and outdoor temperatures.}\label{fig:causal_graph}
\end{center}
\end{figure}

\section{Proofs}\label{ap:proof}
In this section, we provide the proofs of the theorems, propositions, corollaries and lemmas appearing in this paper.

\section*{Proof of Proposition~\ref{prop:orthonormal_existance}}
First, we prove the following lemma and corollary:
\begin{lemma}\label{lem:inprodequiv}
For $c\in\mcl{A}$ and $v\in\mcl{M}$, if $\blacket{v,v}c=\blacket{v,v}$, then $vc=v$ holds.
\end{lemma}
\begin{proof}
If $\blacket{v,v}c=\blacket{v,v}$, the following equalities hold:
\begin{align*}
 &\blacket{vc-v,vc-v}\\
 &\quad=c^*\blacket{v,v}c-c^*\blacket{v,v}-\blacket{v,v}c+\blacket{v,v}=0,
\end{align*}
 which imply $vc=v$.
\end{proof}
\begin{cor}\label{cor:inprodequiv}
If $q\in\mcl{M}$ is normalized, then $q\blacket{q,q}=q$ holds.
\end{cor}
\begin{proof}
Since $\blacket{q,q}$ is a projection, $\blacket{q,q}\blacket{q,q}=\blacket{q,q}$ holds.
Therefore, letting $c=\blacket{q,q}$ and $v=q$ in Lemma~\ref{lem:inprodequiv} completes the proof of the corollary.
\end{proof}

Next, we show the space spanned by an orthonormal system is orthogonally complemented. 
\begin{lemma}\label{lem:orthcomp}
Assume $\mcl{A}=\mat$.
Let $\{q_t\}_{t\in\mcl{T}}$ be an orthonormal system of $\modu$ and let 
$\mcl{V}:=\overline{\{\sum_{t\in\mcl{F}}q_{t}c_t\mid\ \mcl{F}:\mbox{finite subset of }\mcl{T},\ c_t\in\mcl{A}\}}$.
Then, $\mcl{V}+\mcl{V}^{\perp}=\mcl{M}$ holds.
\end{lemma}
\begin{proof}
For $v\in\mcl{M}$, we show $\sum_{t\in\mcl{T}}q_t\blacket{q_t,v}$ converges with respect to $\Vert\cdot\Vert$ as follows:\@
Let $\mcl{T}_0$ be an arbitrary finite set that satisfies $\mcl{T}_0\subseteq\mcl{T}$.
Then, the following equalities hold:
\begin{align*}
&\Bblacket{\sum_{t\in\mcl{T}_0}q_t\blacket{q_t,v},\sum_{t\in\mcl{T}_0}q_t\blacket{q_t,v}}\\
&\qquad=\sum_{t\in\mcl{T}_0}\sum_{s\in\mcl{T}_0}\blacket{q_t,v}^*\blacket{q_t,q_s}\blacket{q_s,v}\nn\\
&\qquad =\sum_{t\in\mcl{T}_0}\blacket{q_t,v}^*\blacket{q_t,q_t}\blacket{q_t,v}\\
&\qquad=\sum_{t\in\mcl{T}_0}\blacket{q_t,v}^*\blacket{q_t,v},
\end{align*}
where the last equality is by Corollary~\ref{cor:inprodequiv}.
Therefore, the following equalities are deduced:
\begin{align*}
0&\le\Bblacket{v-\sum_{t\in\mcl{T}_0}q_t\blacket{q_t,v},v-\sum_{t\in\mcl{T}_0}q_t\blacket{q_t,v}}\\
&=\blacket{v,v}-\sum_{t\in\mcl{T}_0}\blacket{q_t,v}^*\blacket{q_t,v}.
\end{align*}
Let $c_t:=\blacket{q_t,v}^*\blacket{q_t,v}$.
Since $\blacket{v,v}-\sum_{t\in\mcl{T}_0}c_t$ is positive, for $\xi\in\mathbb{C}^n$, 
$\xi^*\blacket{v,v}\xi-\sum_{t\in\mcl{T}_0}\xi^*c_t\xi\ge 0$ holds.
Thus, $\sum_{t\in\mcl{T}}\xi^*c_t\xi$ converges, that is, there exists a limit $\alpha(\xi)\ge 0$ such that $\sum_{t\in\mcl{T}}\xi^*c_t\xi=\alpha(\xi)$.
Since there exists positive $c\in\mathbb{C}^{n\times n}$ such that $\alpha(\xi)=\xi^*c\xi$, 
$\sum_{t\in\mcl{T}}c_t=c$ holds with weak convergence.
For $\mathbb{C}^{n\times n}$, weak convergence is equal to the norm convergence in $\mat$.
Therefore, for $\epsilon>0$, there exists $\mcl{T}'\subseteq\mcl{T}$ such that the following equalities hold for arbitrary finite sets $\mcl{T}_0$ and $\mcl{T}_1$ that satisfy $\mcl{T}'\subseteq\mcl{T}_0$ and $\mcl{T}_1\subseteq\mcl{T}$:
\begin{align*}
&\bigg\Vert\sum_{t\in\mcl{T}_0\setminus\mcl{T}_1}c_t\bigg\Vert_{\mat}=\bigg\Vert\sum_{t\in\mcl{T}_0\setminus(\mcl{T}_0\bigcap\mcl{T}_1)}c_t\bigg\Vert_{\mat}\\
&\quad=\bigg\Vert\sum_{t\in\mcl{T}_0}c_t-\sum_{t\in\mcl{T}_0\bigcap\mcl{T}_1}c_t\bigg\Vert_{\mat}\le\epsilon.
\end{align*}
In the same manner, $\Vert\sum_{t\in\mcl{T}_1\setminus\mcl{T}_0}c_t\Vert\le\epsilon$ holds.
By these inequalities and an equality $c_t=\vert q_t\blacket{q_t,v}\vert^2$, the following inequalities and equalities are derived:
\begin{align*}
& \bigg\Vert\sum_{t\in\mcl{T}_0}q_t\blacket{q_t,v}-\sum_{t\in\mcl{T}_1}q_t\blacket{q_t,v}\bigg\Vert^2\\
&\le 2\bigg\Vert\sum_{t\in\mcl{T}_0\setminus\mcl{T}_1}q_t\blacket{q_t,v}\bigg\Vert^2+2\bigg\Vert\sum_{t\in\mcl{T}_0\setminus\mcl{T}_1}q_t\blacket{q_t,v}\bigg\Vert^2\nn\\
&=2\bigg\Vert\bigg\vert\sum_{t\in\mcl{T}_0\setminus\mcl{T}_1}q_t\blacket{q_t,v}\bigg\vert^2\bigg\Vert_{\mat}\\
&\qquad\qquad+2\bigg\Vert\bigg\vert\sum_{t\in\mcl{T}_0\setminus\mcl{T}_1}q_t\blacket{q_t,v}\bigg\vert^2\bigg\Vert_{\mat}\nn\\
&=2\bigg\Vert\sum_{t\in\mcl{T}_0\setminus\mcl{T}_1}c_t\bigg\Vert_{\mat}+2\bigg\Vert\sum_{t\in\mcl{T}_0\setminus\mcl{T}_1}c_t\bigg\Vert_{\mat}
\le 4\epsilon.
\end{align*}
Thus, $\sum_{t\in\mcl{T}}q_t\blacket{q_t,v}$ is a Cauchy net, and by the completeness of $\modu$, it converges with respect to $\Vert\cdot\Vert$.
Let $P:\mcl{M}\to\mcl{V}$ be a map $v\mapsto\sum_{t\in\mcl{T}}q_t\blacket{q_t,v}$.
Then, for $\sum_{t\in\mcl{F}}q_{t}c_t\in\mcl{V}$, the following equalities hold:
\begin{align*}
& \Bblacket{\sum_{t\in\mcl{F}}q_{t}c_t,v-Pv}\\
&\qquad=\sum_{t\in\mcl{F}}c_t^*\blacket{q_{t},v}-\sum_{s\in\mcl{F}}c_s^*\sum_{t\in\mcl{T}}\blacket{q_{s},q_t}\blacket{q_t,v}\nn\\
&\qquad=\sum_{s\in\mcl{F}}c_s^*\blacket{q_{s},v}-\sum_{s\in\mcl{F}}c_s^*\blacket{q_{s},v}=0,
\end{align*}
which imply $v-Pv\in\mcl{V}^{\perp}$.
As a result, for $v\in\mcl{M}$, $v=Pv+(v-Pv)$, $Pv\in\mcl{V}$ and $v-Pv\in\mcl{V}^{\perp}$ hold, which imply $\mcl{V}+\mcl{V}^{\perp}=\mcl{M}$.
\end{proof}

With Lemma~\ref{lem:orthcomp}, we now show the statement of Proposition~\ref{prop:orthonormal_existance} using Zorn's Lemma.
\begin{proof}[Proof of Proposition~\ref{prop:orthonormal_existance}]
Let $\mcl{S}$ be the set of all orthonormal systems of $\modu$.
We define the partial order on $\mcl{S}$ as the inclusion.
Let $\mcl{S}_0$ be an arbitrary totally ordered subset of $\mcl{S}$.
Then, $\bigcup_{\mcl{V}\in\mcl{S}_0}\mcl{V}$ is an upper bound of 
$\mcl{S}_0$.
Therefore, by Zorn's Lemma, there exists a maximal element, denoted by $\{q_t\}_{t\in\mcl{T}}$, in $\mcl{S}$.
Let $\mcl{V}_0:=\overline{\{\sum_{t\in\mcl{F}}q_{t}c_t\mid\ \mcl{F}:\mbox{finite subset of }\mcl{T},\ c_t\in\mat\}}$.
By Lemma~\ref{lem:orthcomp}, $\mcl{V}_0+\mcl{V}_0^{\perp}=\mcl{M}$ holds.
Thus, if $\mcl{V}_0\neq\mcl{M}$, $\mcl{V}_0^{\perp}\neq \{0\}$ and by Lemma~\ref{prop:normalized_property}, there exists $v\in\mcl{V}_0^{\perp}$ such that it is normalized, which contradicts the maximality of $\mcl{V}_0$.
As a result, $\mcl{V}_0=\mcl{M}$ holds.
\end{proof}

\section*{Proof of Proposition~\ref{prop:min_projection}}
By Lemma~\ref{lem:orthcomp}, $w\in\mcl{M}$ is decomposed into $w=w_1+w_2$, where  $w_1:=Pw\in\mcl{V}$ and $w_2:=w-w_1\in\mcl{V}^{\perp}$.
Let $v\in\mcl{V}$.
Since $w_1-v\in\mcl{V}$, $\blacket{w_2,w_1-v}=0$ holds.
Therefore, the following equalities hold:
\begin{equation}
 \vert w-v\vert^2=\vert w_2+(w_1-v)\vert^2=\vert w_2\vert^2+\vert w_1-v\vert^2\label{eq:deviation},
\end{equation}
which imply $\vert w-v\vert^2-\vert w-w_1\vert^2\ge 0$.
Since $v\in\mcl{V}$ is arbitrary, $w_1$ is a solution of $\min_{v\in\mcl{V}}\vert w-v\vert$.

Moreover, if there exists $w'\in\mcl{W}$ such that $\vert w-w_1\vert^2=\vert w-w'\vert^2$, letting $v=w'$ in Eq.~\eqref{eq:deviation} derives $\vert w-w'\vert^2=\vert w_2\vert^2+\vert w_1-w'\vert^2$, which implies $\vert w_1-w'\vert^2=0$.
As a result, $w_1=w'$ holds and the uniqueness of $w_1$ has been proved. 

\section*{Proof of Lemma~\ref{lem:pdk}}
For $x=[x_1,\ldots,x_m],y=[y_1,\ldots,y_m]\in\mcl{X}^m$, equalities  $[k(x,y)]_{i,j}=\tilde{k}(x_i,y_j)=\overline{\tilde{k}(y_j,x_i)}=[k(y,x)^*]_{i,j}$ hold for $i,j=1,\ldots,m$, which implies $k(x,y)=k(y,x)^*$.
Also, for $n\in\mathbb{N}$, $c_1,\ldots,c_n\in\mat$, $x_1,\ldots,x_n\in\mcl{X}^m$ and $d\in\mathbb{C}^m$, the following equality holds:
\begin{align*}
d^*\sum_{t,s=1}^nc_t^*k(x_t,y_s)c_sd=\sum_{t,s=1}^n\sum_{i,j=1}^m\overline{\tilde{d}_{t,i}}\tilde{k}(x_{t,i},y_{s,j})\tilde{d}_{s,j},
\end{align*}
where $\tilde{d}_t:=c_td$ and $\tilde{d}_{t,i}$ represents the $i$-the element of $\tilde{d}_t$.
Thus, by the positive definiteness of $\tilde{k}$, $d^*\big(\sum_{t,s=1}^nc_t^*k(x_t,y_s)c_s\big)d\ge 0$ holds.
This completes the proof of the proposition.

\section*{Proof of Proposition~\ref{prop:normalized_property}}
Let $\lambda_1\ge\ldots\ge\lambda_{m}\ge 0$ be the eigenvelues of $\blacket{\hat{q},\hat{q}}$, and $m':=\max\{j\mid\ \lambda_j>\epsilon^2\}$.
Since $\blacket{\hat{q},\hat{q}}$ is positive,
there exists an unitary matrix $c$ such that $\blacket{\hat{q},\hat{q}}=c^*\opn{diag}\{\lambda_1,\ldots,\lambda_{m}\}c$.
Also, since $\lambda_1=\Vert\hat{q}\Vert^2>\epsilon^2$, $m'\ge 1$ holds.
Let $\hat{b}:=c^*\opn{diag}\{1/\sqrt{\lambda_1},\ldots,1/\sqrt{\lambda_{m'}},0,\ldots,0\}c$. 
by the definition of $\hat{b}$, $\Vert \hat{b}\Vert_{\mat}=1/\sqrt{\lambda_{m'}}<1/\epsilon$ holds.
Also, the following equalities are derived:
\begin{align*}
\sblacket{\hat{q}\hat{b},\hat{q}\hat{b}}&=\hat{b}^*\blacket{\hat{q},\hat{q}}\hat{b}\nn\\
&=c\opn{diag}(1,\ldots,1,0,\ldots,0)c^*.
\end{align*}
Thus, $\sblacket{\hat{q}\hat{b},\hat{q}\hat{b}}$ is a nonzero orthogonal projection.

In addition, let $b:=c^*\opn{diag}\{\sqrt{\lambda_1},\ldots,\sqrt{\lambda_{m'}},0,\ldots,0\}c$.
Since $\hat{b}b=\opn{diag}(1,\ldots,1,0,\ldots,0)$, $\sblacket{\hat{q},\hat{q}\hat{b}b}=\sblacket{\hat{q}\hat{b}b,\hat{q}\hat{b}b}$ holds, and the following equalities are derived:
\begin{align*}
&\sblacket{\hat{q}-qb,\hat{q}-qb}\\
&=\sblacket{\hat{q}-\hat{q}\hat{b}b,\hat{q}-\hat{q}\hat{b}b}
=\sblacket{\hat{q},\hat{q}}-\sblacket{\hat{q}\hat{b}b,\hat{q}\hat{b}b}\nn\\
&=c\left(\opn{diag}(\lambda_1,\ldots,\lambda_{m})-\opn{diag}(\lambda_1,\ldots,\lambda_{m'},0,\ldots,0)\right)c^*.
\end{align*}
Thus, for $m'\le m-1$, $\Vert \hat{q}-q\hat{b}\Vert=\sqrt{\lambda_{m'+1}}\le\epsilon$ holds, and for $m'=m$, $\Vert \hat{q}-q\hat{b}\Vert=0$ holds, which completes the proof of the proposition.
\section*{Proof of Proposition~\ref{prop:gram-schmidt}}
By Proposition~\ref{prop:normalized_property}, $q_t$ is normalized, and for $\epsilon\ge 0$, there exists $b_t\in\mat$ such that $\Vert \hat{q}_t-q_tb_t\Vert\le \epsilon$.
Therefore, by the definition of $\hat{q}_t$, 
$\Vert w_t-v_t\Vert\le\epsilon$ holds, where $v_t$ is a vector in the space spanned by $\{q_t\}_{t=0}^{\infty}$ which is defined as  $v_t=\sum_{s=1}^{t-1}q_s\blacket{q_s,w_t}-q_tb_t$.
This means the $\epsilon$-neighborhood of the space spanned by $\{q_t\}_{t=1}^{\infty}$ contains $\{w_t\}_{t=1}^{\infty}$.
Next, we show the orthogonality of $\{q_t\}_{t=1}^{\infty}$.
Assume $q_1,\ldots,q_{t-1}$ are orthgonal to each other.
For $s<t$, the following equalities are deduced by Corollary~\ref{cor:inprodequiv}:
\begin{align*}
&\blacket{q_t,q_s}=\hat{b}_t^*\blacket{\hat{q}_t,q_s}\\
&\quad=\hat{b}_t^*\Bblacket{w_t-\sum_{l=1}^{t-1}q_l\blacket{q_l,w_t},q_s}\\
&\quad=\hat{b}_t^*\left(\blacket{w_t,q_s}-\blacket{q_s\blacket{q_s,w_t},q_s}\right)\\
&\quad=\hat{b}_t^*\left(\blacket{w_t,q_s}-\blacket{w_t,q_s}\right)=0.
\end{align*}
Therefore, $q_1,\ldots,q_t$ are also orthogonal to each other, which completes the proof of the proposition.

\section*{Proof of Corollary~\ref{prop:qr}}
The equality $Q=W\mathbf{R}_{\opn{inv}}$ is derived directly from scheme~\eqref{eq:gram-schmidt}.
In addition, by scheme~\eqref{eq:gram-schmidt}, the following equalities hold for $t=1,\ldots,n$:
\begin{align*}
w_t&=\sum_{s=1}^{t-1}q_s\blacket{q_s,w_t}+\hat{q}_t\\
&=\sum_{s=1}^{t-1}q_s\blacket{q_s,w_t}+q_tb_t+\hat{q}_t-q_tb_t\\
&=Q\mathbf{r}_t+\hat{q}_t-q_tb_t,
\end{align*}
where $\mathbf{r}_t\in(\mat)^n$ is the $t$-th column of $\mathbf{R}$ as an $n\times n$ $\mat$-valued matrix.
Therefore, by Proposition~\ref{prop:normalized_property}, $\Vert w_t-Q\mathbf{r}_t\Vert=\Vert \hat{q}_t-q_tb_t\Vert\le\epsilon$ holds for $t=1,\ldots,n$, which implies $\Vert W-Q\mathbf{R}\Vert\le \epsilon$. 
Here, $\Vert W\Vert$ for a $\mat$-linear map $W:\alg^n\to\modu$ is defined as $\Vert W\Vert:=\sup_{\Vert v\Vert_{(\mat)^n}=1}\Vert Wv\Vert$.

\section*{Proof of Proposition~\ref{prop:rank1proj}}
First, we show $\{p_t\}_{t=1}^l$ is an orthonormal system and $\blacket{q_s,q_s}_k$ is rank-one.
The following equalities hold:
\begin{align*}
&\blacket{p_s,p_l}_k\\
&=\Bblacket{\frac1{\sqrt{\sigma_s}}W[\mathrm{v}_s,0,\ldots,0],\frac1{\sqrt{\sigma_l}}W[\mathrm{v}_l,0,\ldots,0]}_k\\
&=\frac1{\sqrt{\sigma_s\sigma_l}}\blacket{[\mathrm{v}_s,0,\ldots,0],W^*W[\mathrm{v}_l,0,\ldots,0]}_k\\
&=\frac1{\sqrt{\sigma_s\sigma_l}}\blacket{[\mathrm{v}_s,0,\ldots,0],[\sigma_l\mathrm{v}_l,0,\ldots,0]}_k.
\end{align*} 
Therefore, $\blacket{p_s,p_l}_k$ is rank one projection for $s=l$ and $\blacket{p_s,p_l}_k=0$ for $s\neq l$.

Next, we show the space spanned by $\{w_t\}_{t=1}^n$ is contained in the space spanned by $\{p_t\}_{t=1}^l$.
Let $w:=\sum_{t=1}^{n}w_tc_t$ for $c_t\in\mat$.
Let $d_t\in\mat$ be a matrix whose first row is $\mathrm{v}_t^*\mathbf{c}$, where 
$\mathbf{c}=[c_1^*,\ldots,c_n^*]^*$.
Then, the following equalities hold:
\begin{align*}
&\Bblacket{w-\sum_{t=1}^lW\mathbf{v}_td_t,w-\sum_{t=1}^lW\mathbf{v}_td_t}_k\\
&=\Bblacket{W\mathbf{c}-\sum_{t=1}^lW\mathbf{v}_td_t,W\mathbf{c}-\sum_{t=1}^lW\mathbf{v}_td_t}_k\\
&=\Bblacket{\mathbf{c}-\sum_{t=1}^l\mathbf{v}_td_t,W^*W\bigg(\mathbf{c}-\sum_{t=1}^l\mathbf{v}_td_t\bigg)}_k\\
&=\Bblacket{\mathbf{c}-\sum_{t=1}^l\mathbf{v}_td_t,\mathbf{V\Sigma V^*}\bigg(\mathbf{c}-\sum_{t=1}^l\mathbf{v}_td_t\bigg)}_k\\
&=\bigg\langle\mathbf{V}^*\mathbf{c}-\sum_{t=1}^l[\mathrm{e}_t,0,\ldots,0]d_t,\\
&\qquad\mathbf{\Sigma}\bigg(\mathbf{V}^*\mathbf{c}-\sum_{t=1}^l[\mathrm{e}_t,0,\ldots,0]d_t\bigg)\bigg\rangle_k\\
&=0,
\end{align*}
where $\mathbf{v}_t:=[\mathrm{v}_t,0,\ldots,0]$ and $\mathrm{e}_t$ is the vector in $\mathbb{C}^{l}$ whose $s$th element is $1$ for $s=t$ and $0$ for $s\neq t$.
Therefore, $w=\sum_{t=1}^lW\mathbf{v}_td_t$ holds and $w$ is contained in the space spanned by $\{p_s\}_{s=1}^l$, which completes the proof of the proposition.

\section*{Proof of Proposition~\ref{prop:pca_generalization}}
If $m=1$, $\phi(x)$ is equal to $\tilde{\phi}(x)$, 
and $p_s$ is reduced to $1/\sqrt{\tilde{\sigma}_s}\tilde{W}\tilde{\mathrm{v}}_s$, with the notations in Appendix~\ref{ap:pca_review}, which is equal to the principal axis $\tilde{p}_s$ defined in Eq.~\eqref{eq:p_axis_rkhs}.

\section*{Proof of Theorem~\ref{prop:pca_ineq}}
First, we consider a maximization problem that is equivalent to minimization problem~\eqref{eq:pca_max}.
The following equalities hold:
\begin{align*}
&\opn{tr}\bigg(\sum_{t=1}^{n} \bigg\vert w_t-\sum_{j=1}^s\hat{p}_j\blacket{\hat{p}_j,w_t}_k\bigg\vert_k^2\bigg)\\
&=\opn{tr}\bigg(\sum_{t=1}^{n} \blacket{w_t,w_t}_k-\sum_{j=1}^s\blacket{w_t,\hat{p}_j}_k\blacket{\hat{p}_j,w_t}_k\bigg)\\
&=\opn{tr}\bigg(\sum_{t=1}^{n} \blacket{w_t,w_t}_k\bigg)\\
&\qquad-\opn{tr}\bigg(\sum_{t=1}^{n}\sum_{j=1}^s\blacket{w_t,\hat{p}_j}_k\blacket{\hat{p}_j,w_t}_k\bigg),
\end{align*}
which imply maximization problem~\eqref{eq:pca_max} is equivalent to the maximization problem of $\opn{tr}\left(\sum_{t=1}^{n}\sum_{j=1}^s\blacket{w_t,\hat{p}_j}_k\blacket{\hat{p}_j,w_t}_k\right)$.
Let $W=Q\mathbf{R}$ be the QR decomposition of $W$ and $\mathbf{C}$ be a $mn\times l$ complex-valued matrix that satisfies $Q^*Q=\mathbf{CC}^*$ and $\mathbf{C}^*\mathbf{C}=\mathbf{I}$.
Let $\mathbf{U}:=\mathbf{C^*RV}\mathbf{\Sigma}^{-1/2}$.
Then, $\mathbf{U\Sigma U^*}=\mathbf{C^*RR^*C}$ holds, and the following equalities hold:
\begin{align*}
 W\mathbf{V}\mathbf{\Sigma}^{-1/2}&=QQ^*Q\mathbf{RV}\mathbf{\Sigma}^{-1/2}\\
 &=Q\mathbf{CC^*RV}\mathbf{\Sigma}^{-1/2}=Q\mathbf{CU}.
\end{align*}
Therefore, $p_s$, which is defined as $1/\sqrt{\sigma_s}W[\mathrm{v}_s,0,\ldots,0]$, is equal to $Q\mathbf{C}[\mathrm{u}_s,0,\ldots,0]$, where $\mathrm{u}_s$ is the $s$-th column of $\mathbf{U}$.
Let $\hat{\mathbf{p}}_j:=Q^*\hat{p}_j$ for normalized $\hat{p}_j\in\{\sum_{t=1}^{j-1}p_tc_t\mid\ c_t\in\mat\}^{\perp_k}$ where the rank of $\blacket{\hat{p}_j,\hat{p}_j}_k$ is one.
By the equality $Q^*QQ^*=Q^*$, and by Proposition~\ref{prop:rank1proj},
\begin{align*}
 \hat{\mathbf{p}}_j&=Q^*\hat{p}_j=Q^*QQ^*\hat{p}_j=Q^*Q\mathbf{C}\sum_{t=1}^l[\mathrm{u}_t,0\ldots,0]c_{t,j}\\
 &=\mathbf{C}\sum_{t=1}^l[\mathrm{u}_t,0\ldots,0]c_{t,j},  
\end{align*}
for some $c_{t,j}\in\mat$ holds.
By the orthogonality of $\hat{p}_j$, in fact, the first $j-1$ terms of the sum $\sum_{t=1}^l[\mathrm{u}_t,0\ldots,0]c_{t,j}$ is equal to $0$.
Indeed, the following equalities hold:
\begin{align*}
&\blacket{\hat{p}_j,p_t}_k
=\blacket{\hat{p}_j,Q\mathbf{C}[\mathrm{u}_t,0,\ldots,0]}_k\\
&=\blacket{Q^*\hat{p}_j,\mathbf{C}[\mathrm{u}_t,0,\ldots,0]}_k\\
&=\Bblacket{\mathbf{C}\sum_{s=1}^l[\mathrm{u}_s,0\ldots,0]c_{t,j},\mathbf{C}[\mathrm{u}_t,0,\ldots,0]}_k\\
&=\blacket{\mathbf{C}[\mathrm{u}_t,0\ldots,0]c_{t,j},\mathbf{C}[\mathrm{u}_t,0,\ldots,0]}_k\\
&=c_{t,j}^*\opn{diag}\{1,0,\ldots,0\}.
\end{align*}
Since $\blacket{\hat{p}_j,p_t}_k=0$ for $t=1,\ldots,j-1$, the first row of $c_{t,j}$ is equal to $0$ for $t=1,\ldots,j-1$, which implies $\mathbf{C}\sum_{t=1}^l[\mathrm{u}_t,0\ldots,0]c_{t,j}=\mathbf{C}\sum_{t=j}^l[\mathrm{u}_t,0\ldots,0]c_{t,j}$.
Moreover, since $\hat{\mathbf{p}}_j^*\hat{\mathbf{p}}_j=\blacket{QQ^*\hat{p}_j,\hat{p}_j}_k=\blacket{QQ^*\hat{p}_j,\hat{p}_j\blacket{\hat{p}_j,\hat{p}_j}_k}_k=\blacket{QQ^*\hat{p}_j,\hat{p}_j}_k\blacket{\hat{p}_j,\hat{p}_j}_k$ holds and $\blacket{\hat{p}_j,\hat{p}_j}_k$ is rank-one, $\hat{\mathbf{p}}_j^*\hat{\mathbf{p}}_j$ is rank-one, 
and by Cauchy-Schwarz inequality, $\Vert \hat{\mathbf{p}}_j^*\hat{\mathbf{p}}_j\Vert_{\mat}=\Vert\blacket{QQ^*\hat{p}_j,\hat{p}_j}_k\Vert_{\mat}\le 1$ holds.
Thus, $\hat{\mathbf{p}}_j\hat{\mathbf{p}}_j^*$ is also rank-one and $\Vert \hat{\mathbf{p}}_j\hat{\mathbf{p}}_j^*\Vert_{(\mat)^n}\le 1$ holds.
As a result, there exists $\hat{\mathrm{p}}_j\in\mathbb{C}^{mn}$, which is the linear combination of $\mathbf{C}\mathrm{u}_j,\ldots,\mathbf{C}\mathrm{u}_l$, and satisfies $\hat{\mathrm{p}}_j^*\hat{\mathrm{p}}_j\le 1$ and $\hat{\mathbf{p}}_j\hat{\mathbf{p}}_j^*=\hat{\mathrm{p}}_j\hat{\mathrm{p}}_j^*$.
Let $\hat{\mathrm{p}}_j=\mathbf{C}\sum_{t=j}^l\mathrm{u}_tc_{t,j}$ for some $c_{t,j}\in\mathbb{C}$.
Since $\hat{\mathrm{p}}_j^*\hat{\mathrm{p}}_j\le 1$, $\sum_{t=j}^l\vert c_{t,j}\vert^2\le 1$ holds.
Also, since $\mathbf{C}^*\mathbf{R}\mathbf{R}^*\mathbf{C}=\mathbf{U\Sigma}\mathbf{U}^*$, $Q^*Q\mathbf{R}\mathbf{R}^*Q^*Q=\mathbf{CU\Sigma}\mathbf{U}^*\mathbf{C}^*$ holds.
Thus, the following equalities are derived:
\begin{align*}
&\sum_{t=1}^{n} \sum_{j=1}^s\opn{tr}\big(\blacket{w_t,\hat{p}_j}_k\blacket{\hat{p}_j,w_t}_k\big)\\
&=\sum_{t=1}^{n} \sum_{j=1}^s\opn{tr}\big(\blacket{QQ^*w_t,\hat{p}_j}_k\blacket{\hat{p}_j,QQ^*w_t}_k\big)\\
&=\sum_{t=1}^{n} \sum_{j=1}^s\opn{tr}\big(\mathbf{e}_t^*W^*Q\hat{\mathrm{p}}_j\hat{\mathrm{p}}_j^*Q^*W\mathbf{e}_t\big)\\
&=\sum_{t=0}^{n} \sum_{j=1}^s\hat{\mathrm{p}}_j^*Q^*W\mathbf{e}_t\mathbf{e}_t^*W^*Q\hat{\mathrm{p}}_j\\
&=\sum_{j=1}^s\hat{\mathrm{p}}_j^*Q^*WW^*Q\hat{\mathrm{p}}_j=\sum_{j=1}^s\hat{\mathrm{p}}_j^*Q^*Q\mathbf{R}\mathbf{R}^*Q^*Q\hat{\mathrm{p}}_j\\
&=\sum_{j=1}^s\hat{\mathrm{p}}_j^*\mathbf{CU\Sigma}\mathbf{U}^*\mathbf{C}^*\hat{\mathrm{p}}_j\\
&=\sum_{j=1}^s\bigg(\sum_{t=j}^l\overline{c_{t,j}}\mathrm{e}_t^*\bigg)\mathbf{\Sigma}\bigg(\sum_{t=j}^lc_{t,j}\mathrm{e}_t\bigg)\\
&=\sum_{j=1}^s\sum_{t=j}^l\vert c_{t,j}\vert^2\sigma_t\le \sum_{j=1}^s\sigma_j,
\end{align*}
where $\mathbf{e}_t$ is the vector in $(\mat)^m$ whose $s$-th element is the identity for $s=t$ and $0$ for $s\neq t$.
The last inequality becomes the equality if $c_{j,j}=1$ and $c_{t,j}=0$ for $t\neq j$, which completes the proof of the proposition.

\section*{Proof of Theorem~\ref{prop:min_sol}}
Let $\hat{K}_0\in\mcl{L}(\mcl{W}_T)$ be a linear operator that satisfies $\hat{K}_0\phi(x_t)=Q_TQ_T^*\phi(x_{t+1})$ for $t=0,\ldots,T-1$.
$\hat{K}_0$ is well-defined because $\{\phi(x_t)\}_{t=0}^{T-1}$ is linearly independent.
Since $\epsilon=0$, by Corollary~\ref{cor:gram-schmidt}, $\mcl{W}_T$ is equal to the space spanned by $\{q_t\}_{t=0}^{T-1}$.
Thus, by proposition~\ref{prop:qr}, the projection of a vector $u\in\modu_k$ onto $\mcl{W}_T$ is represented as $Q_TQ_T^*u$, and $\vert u-v\vert_k-\vert u-Q_TQ_T^*u\vert_k\ge 0$ holds for any $v\in\mcl{W}_T$.
For simplicity, in the following, we denote $c-d\ge 0$ as $c\ge d$ or $d\le c$ for $c,d\in\mat$.
Therefore, for any $\hat{K}\in\mcl{L}(\mcl{W}_T)$, the following inequalities hold:
\begin{align*}
&\sum_{t=0}^{T-1}\big\vert \hat{K}_0\phi(x_t)-\phi(x_{t+1})\big\vert_k^2\\
&\qquad=\sum_{t=0}^{T-1}\vert Q_TQ_T^*\phi(x_{t+1})-\phi(x_{t+1})\vert_k^2\\
&\qquad\le\sum_{t=0}^{T-1}\big\vert \hat{K}\phi(x_t)-\phi(x_{t+1})\big\vert_k^2,
\end{align*}
which implies $\hat{K}_0$ is a solution of minimization problem~\eqref{eq:min_krylov}.

Assume $\hat{K}_1\in\mcl{L}(\mcl{W}_T)$ satisfies $\sum_{t=0}^{T-1}\vert \hat{K}_1\phi(x_t)-\phi(x_{t+1})\vert_k^2=\sum_{t=0}^{T-1}\vert \hat{K}_0\phi(x_t)-\phi(x_{t+1})\vert_k^2$.
Since $\hat{K}_0\phi(x_t)=Q_TQ_T^*\phi(x_{t+1})=\phi(x_{t+1})$ for $t=0,\ldots,T-2$, $\sum_{t=0}^{T-1}\vert \hat{K}_0\phi(x_t)-\phi(x_{t+1})\vert_k^2=\vert Q_TQ_T^*\phi(x_{T})-\phi(x_{T})\vert_k^2$ holds.
Thus, if there exists $t\in\{0,\ldots,T-2\}$ such that $\vert \hat{K}_1\phi(x_t)-\phi(x_{t+1})\vert_k^2=c$ for some positive $c\neq 0$, then
\begin{align*}
&\vert Q_TQ_T^*\phi(x_{T})-\phi(x_{T})\vert_k^2\\
&\qquad=\sum_{t=0}^{T-1}\vert \hat{K}_0\phi(x_t)-\phi(x_{t+1})\vert_k^2\\
&\qquad=\sum_{t=0}^{T-1}\vert \hat{K}_1\phi(x_t)-\phi(x_{t+1})\vert_k^2\\
&\qquad\ge c+\vert \hat{K}_1\phi(x_{T-1})-\phi(x_{T})\vert_k^2,
\end{align*}
which contradicts Proposition~\ref{prop:min_projection}.
Therefore, $\vert \hat{K}_1\phi(x_t)-\phi(x_{t+1})\vert_k^2=0$ for $t=0,\ldots,T-2$ holds, which implies $\hat{K}_1\phi(x_t)=\phi(x_{t+1})=Q_TQ_T^*\phi(x_{t+1})$ for $t=0,\ldots,T-2$.
Also, 
$\vert \hat{K}_1\phi(x_{T-1})-\phi(x_{T})\vert_k^2=\vert Q_TQ_T^*\phi(x_{T})-\phi(x_{T})\vert_k^2$ holds,
and by the uniqueness of $Q_TQ_T^*\phi(x_{T})$, the relation $\hat{K}_1\phi(x_{T-1})=Q_TQ_T^*\phi(x_{T})$ is derived.
As a result, $\hat{K}_1\phi(x_t)=Q_TQ_T^*\phi(x_{t+1})$ is derived for $t=0,\ldots,T-1$, and thus, $\hat{K}_1=\hat{K}_0$ holds, which implies $\hat{K}_0$ is the unique solution of minimization problem~\eqref{eq:min_krylov}.

In addition, by the definition of $K$, $KW_T=[\phi(x_1),\ldots,\phi(x_{T})]$ holds.
As a result, for $w\in\mcl{W}_T$, $\hat{K}_0$ satisfies the following equalities:
\begin{align*}
&\hat{K}_0w=\hat{K}_0Q_TQ_T^*w=\hat{K}_0W_T\mathbf{R}_{\opn{inv},T}Q_T^*w\nn\\
&=Q_TQ_T^*[\phi(x_1),\ldots,\phi(x_{T})]\mathbf{R}_{\opn{inv},T}Q_T^*w\nn\\
&=Q_TQ_T^*KQ_TQ_T^*w,
\end{align*}
which imply $\hat{K}_0=Q_TQ_T^*KQ_TQ_T^*$ and $Q_T^*KQ_T=Q_T^*[\phi(x_1),\ldots,\phi(x_{T})]\mathbf{R}_{\opn{inv},T}$.

\section*{Proof of Proposition~\ref{prop:pf_conv}}
First, we show $\{q_t\}_{t=0}^{\infty}$ is an orthnormal basis of $\modu_k$.
For $v\in\mcl{M}_k$, there exist $v_i\in\mcl{M}_{k,0}\ (i=1,2,\ldots)$ such that $v=\lim_{i\to\infty}v_i$.
We represent $v_i$ as $v_i=\sum_{t=1}^{n_i}\phi(x_{t,i})c_{t,i}$ with some $x_{t,i}\in\mcl{X}^m$ and $c_{t,i}\in\mat$. Since $\mcl{Y}$ is dense in $\mcl{X}^m$, there exists $x_{t,i,j}\in\mcl{Y}$ such that $\lim_{j\to\infty}x_{t,i,j}=x_{t,i}$.
Since $\tilde{k}$ is continuous, $k$ is also continuous.
Thus, $\phi:\mcl{X}^m\to\mcl{M}_k$ is also continuous with respect to $\Vert\cdot\Vert_k$.
Therefore, $v=\lim_{i\to\infty}\lim_{j\to\infty}v_{i,j}$ holds, where $v_{i,j}:=\sum_{t=1}^{n_i}\phi(x_{t,i,j})c_{t,i}$, which implies 
the space spanned by $\{\phi(x_t)\}_{t=1}^{\infty}$ is dense in $\modu_k$.
As a result, by Corollary~\ref{cor:gram-schmidt}, $\{q_t\}_{t=1}^{\infty}$ is an orthonormal basis of $\modu_k$.

Next, we show $Q_TQ_T^*v$ for $v\in\modu_k$ converges to $v$ as $T\to\infty$.
Let $\mcl{V}:=\overline{\{\sum_{t=0}^{n}q_tc_t\mid\ n\in\mathbb{N},\ c_t\in\mat\}}$.
Since $\{q_t\}_{t=1}^{\infty}$ is an orthonormal basis of $\modu_k$, $\mcl{V}$ is equal to $\modu_k$.
By the proof of Lemma~\ref{lem:orthcomp}, for $v\in\mcl{M}_k$, $\lim_{T\to\infty}Q_TQ_T^*v=\sum_{t=0}^{\infty}q_t\blacket{q_t,v}_k$ exists and $\sum_{t=0}^{\infty}q_t\blacket{q_t,v}_k-v\in\mcl{V}^{\perp_k}$ holds.
Since $\mcl{V}^{\perp_k}=\modu_k^{\perp_k}=\{0\}$, $\lim_{T\to\infty}Q_TQ_T^*v=v$ holds.

As a result, if $K$ is bounded, the following inequalities are derived for $v\in\modu_k$:
\begin{align*}
&\Vert Kv-Q_TQ_T^*KQ_TQ_T^*v\Vert_k\\
&\le \Vert Kv-Q_TQ_T^*Kv\Vert_k\\
&\qquad\qquad+\Vert Q_TQ_T^*Kv-Q_TQ_T^*KQ_TQ_T^*v\Vert_k\\
&\le \Vert Kv-Q_TQ_T^*Kv\Vert_k+\Vert Q_TQ_T^*K\Vert_k\Vert v-Q_TQ_T^*v\Vert_k,
\end{align*}
and it converges to $0$ as $T\to\infty$.
Therefore, $Q_TQ_T^*KQ_TQ_T^*$ converges strongly to $K$.

\section*{Proof of Proposition~\ref{prop:abnormality}}
Since $Q_T{\mathbf{K}}_TQ_T^*\phi(x_{S-1})=\sum_{t=0}^{T-1}\phi(x_t)c_t$, the $(j,j)$ element of the $\alg$-valued inner product between $\phi(x_S)$ and $Q_T{\mathbf{K}}_TQ_T^*\phi(x_{S-1})$ in $\modu_k$ is represented as $\sum_{t=0}^{T-1}\sum_{i=1}^{m}\tilde{k}(x_{S,j},x_{t,i})(c_t)_{i,j}$, which 
is equal to the inner product of $\sum_{t=0}^{T-1}\sum_{i=1}^{m}(c_t)_{i,j}\tilde{\phi}(x_{t,i})$ and $\tilde{\phi}(x_{S,j})$ in $\hil_{\tilde{k}}$.
Moreover, the $(j,j)$ element of prediction error~\eqref{eq:error_rkhm} is represented as 
\begin{align*}
&\tilde{k}(x_{S,j},x_{S,j})-\mu_j-\overline{\mu}_j\\
&\qquad+\sum_{s,t=0}^{T-1}\sum_{l,i=1}^m\overline{(c_s)}_{l,j}\tilde{k}(x_{s,l},x_{t,i})(c_t)_{i,j}\\
&=\bigg\Vert \tilde{\phi}(x_{S,j})-\sum_{t=0}^{T-1}\sum_{i=1}^n(c_t)_{i,j}\tilde{\phi}(x_{t,i})\bigg\Vert_{\tilde{k}}^2,
\end{align*}
where $\mu_j=\sum_{t=0}^{T-1}\sum_{i=1}^{m}\tilde{k}(x_{S,j},x_{t,i})(c_t)_{i,j}$.
This completes the proof of the proposition.

\section*{Proof of Proposition~\ref{prop:eigsum}}
Since $Q_T^*Q_T{\mathbf{K}}_T=Q_T^*Q_TQ_T^*KQ_T={\mathbf{K}}_T$, the following equalities holds:
\begin{align*}
&\Bblacket{Q_T\mathbf{K}_T^sQ_T^*\phi(x_0),Q_T\mathbf{K}_T^sQ_T^*\phi(x_0)}_k\\
&\quad=\Bblacket{Q_T\mathbf{K}_T^s\sum_{t=1}^{mT}\mathbf{v}_tc_t,Q_T\mathbf{K}_T^s\sum_{t=1}^{mT}\mathbf{v}_tc_t}_k\nn\\
&\quad=\Bblacket{\mathbf{K}_T^s\sum_{t=1}^{mT}\mathbf{v}_tc_t,\mathbf{K}_T^s\sum_{t=1}^{mT}\mathbf{v}_tc_t}_{(\mat)^T}\nn\\
&\quad=\sum_{t,l=1}^{nT}\blacket{\mathbf{v}_ta_t^sc_t,\mathbf{v}_{l}a_{l}^sc_{l}}_{(\mat)^T}\\
&\quad=\sum_{t,l=1}^{nT}c_t^*(a_t^*)^s\blacket{\mathbf{v}_t,\mathbf{v}_{l}}_{(\mat)^T}a_{l}^sc_{l}.
\end{align*}
By the definitions of $a_t$ and $\mathbf{v}_t$, $(a_t^*)^s\blacket{\mathbf{v}_t,\mathbf{v}_{l}}a_{l}^s=\opn{diag}\{\overline{{\lambda}}_t^s\blacket{\mathrm{v}_t,\mathrm{v}_l}{\lambda}_l^s,0\ldots,0\}$ holds.
Therefore, terms $c_t^*(a_t^*)^s\blacket{\mathbf{v}_t,\mathbf{v}_{t}}a_{t}^sc_t$ with ${\lambda}_t$ that satisfy $\vert {\lambda}_t\vert=1$ in Eq.~\eqref{eq:eigsum} are invariant with respect to $s$.

\section{Pseudo-codes}\label{ap:psudocode}
We provide pseudo-codes of QR decomposition described in Section~\ref{sec:gram-schmidt} and kernel PCA described in Section~\ref{sec:pca}, respectively.
For a matrix $\mathbf{G}$, the $(i,j)$-element of $\mathbf{G}$ is denoted as $g_{i,j}$ and the $j$-th column of $\mathbf{G}$ is denoted as $\mathbf{g}_j$.
\begin{algorithm}[H]
\caption{QR decomposition in RKHMs}\label{al:qr}
 \begin{algorithmic}[1]
 \REQUIRE Samples $x_1,\ldots,x_n$, $\epsilon\ge 0$
 \ENSURE $\mathbf{R}$, $\mathbf{R}_{\opn{inv}}$
  \FOR{$t=1,\ldots,n$}
  \FOR{$i=1,\ldots,n$}
  \IF{$i<t$}
  \STATE $r_{i,t}=k(x_i,x_t)-\hat{b}_{i,i}\sum_{j=1}^{i-1}r_{j,i}^*r_{j,t}$
  \ELSIF{$i=t$}
  \STATE $d=k(x_t,x_t)-\sum_{j=1}^{t-1} r_{j,t}^*r_{j,t}$
  \STATE Compute the eigenvalue decomposition of $d$, $u\opn{diag}\{\lambda_1,\ldots,\lambda_m\}u^*$ where $\lambda_1\ge\ldots\ge \lambda_m$
  \IF{$\lambda_1>\epsilon^2$}
  \STATE $r_{t,t}=u\opn{diag}\{\sqrt{\lambda_1},\ldots,\sqrt{\lambda_{m'}},0,\ldots,0\}u^*$ for $m'$ which satisfies $\lambda_{m'}>\epsilon^2$ and $\lambda_{m'+1}\le\epsilon^2$
  \ELSE
  \STATE $r_{t,t}=0$
  \ENDIF
  \STATE $b_{t,t}=r_{t,t}$
  \STATE $\hat{b}_{t,t}=u\opn{diag}\{1/\sqrt{\lambda_1},\ldots,1/\sqrt{\lambda_{m'}},0,\ldots,0\}u^*$
  \ELSE
  \STATE $r_{i,t}=0$
  \ENDIF
  \ENDFOR
  \ENDFOR
  \STATE $\mathbf{R}_{\opn{inv}}=\hat{\mathbf{B}}(\mathbf{I}+(\mathbf{R}-\mathbf{B})\hat{\mathbf{B}})^{-1}$
 \end{algorithmic}
\end{algorithm}

\begin{algorithm}[H]
\caption{Kernel PCA with RKHMs}\label{al:pca}
 \begin{algorithmic}[1]
 \REQUIRE Samples $x_1,\ldots,x_n$
 \ENSURE The coefficients of the $s$-th principal components of $\phi(x_1),\ldots,\phi(x_n)$, $c_1,\ldots,c_n$
 \FOR{$s=1,\ldots,n$}
 \FOR{$t=1,\ldots,n$}
 \STATE $g_{s,t}=k(x_s,x_t)$
 \ENDFOR
 \ENDFOR
 \STATE Compute the eigenvalue decomposition of $\mathbf{G}$, $\mathbf{V}\opn{diag}\{\sigma_1,\ldots,\sigma_{l}\}\mathbf{V}^*$ where $\sigma_1\ge\ldots\ge \sigma_{l}>0$
 \FOR{$t=1,\ldots,n$}
 \STATE $c_t=1/\sqrt{\sigma_s}[\mathrm{v}_s,0,\ldots,0]^*\mathbf{g}_t$
 \ENDFOR
 \end{algorithmic}
\end{algorithm}


\end{document}